\documentclass[journal]{IEEEtran}

\usepackage{graphicx,epsfig,amsfonts,amsmath,amssymb,url,ifthen}
\usepackage{times}
\usepackage{bm}
\usepackage{multirow}
\usepackage{subfigure}
\usepackage{algorithm}
\usepackage{algorithmic}
\usepackage[justification=centering]{caption}

\usepackage{amsopn}
\usepackage{cases}
\usepackage{url}
\usepackage{epsfig}
\usepackage{graphicx,subfig}
\usepackage{graphicx}
\usepackage{subfigure}
\usepackage{cite}
\usepackage{amsthm}
\usepackage{CJK}
\usepackage{color}
\usepackage{makecell}
\usepackage{enumitem}
\usepackage{multirow,tabularx}

\graphicspath{{./Figures/}}

\newcommand{\tx}{\bm{x}}
\newcommand{\ty}{\bm{y}}

\newcommand{\tb}{\bm{b}}

\newcommand{\tH}{\textbf{H}}
\newcommand{\tD}{\textbf{D}}

\newcommand{\tA}{\textbf{A}}

\newcommand{\tI}{\textbf{I}}

\newcommand{\tW}{\textbf{W}}

\newcommand{\tTheta}{\mathbf{\Theta}}
\newcommand{\tv}{\bm{v}}
\newcommand{\tn}{\bm{n}}

\newcommand{\tu}{\bm{u}}

\DeclareMathOperator*{\argmin}{argmin}
\DeclareMathOperator*{\argmax}{argmax}

\newtheorem{thm}{Theorem}

\title{Denoising Prior Driven Deep Neural Network for Image Restoration}

\author{ Weisheng~Dong, Peiyao Wang, Wotao Yin, Guangming Shi, Fangfang Wu, and Xiaotong Lu}

\newboolean{doublecolumn}
\setboolean{doublecolumn}{true}

\begin{document}
\maketitle

\ifthenelse {\boolean{doublecolumn}} {} {\baselineskip=.8cm}

\begin{abstract}

Deep neural networks (DNNs) have shown very promising results for various image restoration (IR) tasks. However, the design of network architectures remains a major challenging for achieving further improvements. While most existing DNN-based methods solve the IR problems by directly mapping low quality images to desirable high-quality images, the observation models characterizing the image degradation processes have been largely ignored. In this paper, we first propose a denoising-based IR algorithm, whose iterative steps can be computed efficiently. Then, the iterative process is unfolded into a deep neural network, which is composed of multiple denoisers modules interleaved with back-projection (BP) modules that ensure the observation consistencies. A convolutional neural network (CNN) based denoiser that can exploit the multi-scale redundancies of natural images is proposed. As such, the proposed network not only exploits the powerful denoising ability of DNNs, but also leverages the prior of the observation model. Through end-to-end training, both the denoisers and the BP modules can be jointly optimized. Experimental results on several IR tasks, e.g., image denoisig, super-resolution and deblurring show that the proposed method can lead to very competitive and often state-of-the-art results on several IR tasks, including image denoising, deblurring and super-resolution.

\end{abstract}


\IEEEpeerreviewmaketitle

\section{Introduction}

Image restoration (IR) aiming to reconstruct a high quality image from its low quality observation has many important applications, such as low-level image processing, medical imaging, remote sensing, surveillance, etc. Mathematically, IR problem can be expressed as $\ty=\tA\tx+\tn$, where $\ty$ and $\tx$ denote the degraded image and the original image, respectively, $\tA$ denotes the degradation matrix relating to an imaging/degradation system, and $\tn$ denotes the additive noise. Note that for different settings of $\tA$, different IR problems can be expressed. For example, the IR problem is a denoising problem \cite{KSVD, BM3D, CSR, LASSC, WNNM} when $\tA$ is an identical matrix and becomes a deblurring problem \cite{Dong:TIP11,IDDBM3D,NCSR,Dong:IJCV15} when $\tA$ is a blurring matrix/operator, or a super-resolution problem \cite{TVSR, Yang:SR08, NCSR, Gao:TIP12} when $\tA$ is a subsampling matrix/operator. Essentially, restoring $\tx$ from $\ty$ is a challenging ill-posed inverse problem. In the past a few decades, the IR problems have been extensively studied. However, they still remain as an active research area.

Generally, existing IR methods can be classified into two main categories, i.e., model-based methods \cite{Osher:TV05,KSVD,ISTA,Mairal:TIP08,NCSR,Zoran:ICCV11,Dong:IJCV15,Roth:IJCV09,Yu:TIP12} and learning-based methods \cite{Freeman:02,Aplus,Schmidt:CVPR14,SRCNN,Wang:CVPR15,Zhang:TIP17}. The model-based methods attack this problem by solving an optimization problem, which is often constructed from a Bayesian perspective. In the Bayesian setting, the solution is obtained by maximizing the posterior $P(\tx|\ty)$, which can be formulated as
\begin{equation}
\tx   =  \argmax_{\tx} \log P(\tx|\ty) = \argmax_{\tx} \log P(\ty|\tx) + \log P(\tx), \label{MAP}
\end{equation}
where $\log P(\ty|\tx)$ and $\log P(\tx)$ denote the data likelihood and the prior terms, respectively. For additive Gaussian noise, $P(\ty|\tx)$ corresponds to the $\ell_2$-norm data fidelity term, and the prior term $P(\tx)$ characterizes the prior knowledge of $\tx$ in a probability setting. Formally, Eq. (\ref{MAP}) can be rewritten as
\begin{equation}
\tx = \argmin_{\tx} ||\ty-\tA\tx||_2^2 + \lambda J(\tx), \label{Obj_fun}
\end{equation}
where $J(\tx)$ denotes the regularizer associated with the prior term $P(\tx)$. Then, the desirable solution is the one that minimizes both the $\ell_2$-norm data fidelity term and the regularization term weighted by parameter $\lambda$. Clearly, the regularization term plays a critical role in searching for high-quality solutions. Numerous regularizers have been developed, ranging from the well-known total variation (TV) regularizer \cite{Osher:TV05}, the sparsity-base regularizers with off-the-shelf transforms or learned dictionaries \cite{KSVD,ISTA,CSR,Mairal:TIP08}, to the nonlocal self-similarity (NLSS) inspired regularizers \cite{NLM,BM3D,NCSR}. The TV regularizer is good at characterizing the piecewise constant signals but unable to model more complex image edges and textures. The sparsity-based techniques are more effective in representing local image structures with a few elemental structures (called atoms) from an off-the-shelf transformation matrix (e.g., DCT and Wavelets) or a learned dictionary. Indeed, the IR community has witnessed a flurry of sparsity-based IR methods \cite{KSVD,CSR,Mairal:TIP08,Yang:SR08} in the past decade. Motivated by the fact that natural images often contain rich repetitive structures, nonolocal regularization techniques \cite{BM3D,NCSR,LASSC,WNNM} combining the NLSS with the sparse representation and low-rank approximation, have shown significant improvements over their local counterparts. Using those carefully designed prior, significant progresses of IR have been achieved. In addition to these explicitly regularized IR methods, denoising-based IR methods have also been proposed \cite{PPP:13,Brifman:ICIP16,Teodoro:ICIP16,RED:17,Chan:17}. In these methods, the original optimization problem is decoupled into two separated subproblems - one for dealing with the data fidelity term and the other for the regularization term, yielding simpler optimization problems. Specifically, the subproblem related to the regularization is a pure denoising problem, and thus other more complex denoising methods that cannot be expressed as regularization terms can also be adopted, e.g., BM3D \cite{BM3D}, NCSR \cite{NCSR} and GMM \cite{Zoran:ICCV11} methods.

Different from the model-based methods that rely on a carefully designed prior, the learning-based IR methods learn mapping functions to infer the missing high-frequency details or desirable high-quality images from the observed image. In the past decade, many learning-based image super-resolution methods \cite{Freeman:02,Aplus,SRCNN,FSRCNN,VDSR}, \cite{Tai:CVPR17,Han:CVPR18,Tong:ICCV17,Zhang:TIP18,Zhang:TIP17} have been proposed, where mapping functions from the low-resolution (LR) patches to high-resolution (HR) patches are learned. Inspired by the great successes of the deep convolution neural network (DCNN) for image classification \cite{Alexnet,Resnet,Liu:PR17,Liu:Arxiv18}, the DCNN models have also been successfully applied to image IR tasks, e.g., SRCNN \cite{SRCNN}, FSRCNN \cite{FSRCNN}, VDSR \cite{VDSR} and EDSR \cite{EDSR} for image super-resolution, and TNRD \cite{TNRD}, DnCNN \cite{Zhang:TIP17} and MemNet \cite{Tai:ICCV17} for image denoising. In these methods, a DCNN is used to learn the mapping function from the degraded images to the original images. Due to its powerful representation ability, the DCNN based methods have shown better IR performances than conventional optimization-based IR methods in various IR tasks \cite{SRCNN,VDSR,TNRD,Tai:CVPR17,Han:CVPR18,Tong:ICCV17,Zhang:TIP18,Zhang:TIP17}.
Though the DCNN models have shown promising results, the DCNN methods lack flexibilities in adapting to different image recovery tasks, as the data likelihood term has not been explicitly exploited. To address this issue, hybrid IR methods that combine the optimization-based methods and DCNN denoisers have been proposed. In \cite{Zhang:CVPR17,Meinhardt:ICCV17}, a set of DCNN models are pre-trained for image denoising task and are integrated into the optimization-based IR framework for different IR tasks. Compared with other optimization-based methods, the integration of the DCNN models has advantages in exploiting the large training dataset and thus leads to superior IR performance. Similar idea has also been exploited in the autoencoder-based IR method \cite{Bigdeli:arXiv}, where denoising autoencoders are pre-trained as a natural image prior and a regularzer based on the pre-trained autoencoder is proposed. The resulting optimization problem is then iteratively solved by gradient descent. Despite the effectiveness of the methods \cite{Zhang:CVPR17,Bigdeli:arXiv}, they have to iteratively solve optimization problems, and thus their computational complexities are high. Moreover, the CNN and autoencoder models adopted in \cite{Zhang:CVPR17,Bigdeli:arXiv} are pre-trained and cannot be jointly optimized with other algorithm parameters.

In this paper, we propose a denoising prior driven deep network to take advantages of both the optimization- and discriminative learning-based IR methods. First, we propose a denoising-based IR method, whose iterative process can be efficiently carried out. Then, we unfold the iterative process into a feed-forward neural network, whose layers mimic the process flow of the proposed denoising-based IR algorithm. Moreover, an effective DCNN denoiser that can exploit the multi-scale redundancies is proposed and plugged into the deep network. Through end-to-end training, both the DCNN denoisers and other network parameters can be jointly optimized. Experimental results show that the proposed method can achieve very competitive and often state-of-the-art results on several IR tasks, including image denoising, deblurring and super-resolution.

\section{Related Work}

We briefly review the IR methods, i.e., the denoising-based IR methods and the discriminative learning-based IR methods, which are related to the proposed method.

\subsection{Denoising-based IR methods}

Instead of using an explicitly expressed regularizer, denoising-based IR methods \cite{PPP:13} allow the use of a more complex image prior by decoupling the optimization problem of Eq. (\ref{Obj_fun}) into two subproblems, one for the data likelihood term and the other for the prior term. By introducing an auxiliary variable $\tv$, Eq. (\ref{Obj_fun}) can be rewritten as
\begin{equation}
(\tx,\tv) = \argmin_{\tx, \tv} \frac{1}{2}||\ty-\tA\tx||_2^2 + \lambda J(\tv), s.t.~\tx = \tv. \label{den-based}
\end{equation}
In \cite{PPP:13,Chan:17}, the ADMM technique is used to convert the above equally constrained optimization problem into two subproblems
\begin{equation}
\begin{split}
&\tx^{(t+1)} = \argmin_{\tx} \frac{1}{2}||\ty-\tA\tx||_2^2 + \frac{\mu}{2}||\tx-\tv^{(t)}+\tu^{(t)}||_2^2, \\
&\tv^{(t+1)} = \argmin_{\tv} \frac{\mu}{2}||\tx^{(t+1)}-\tv+\tu^{(t)}||_2^2 + \lambda J(\tv),
\end{split}
\end{equation}
where $\tu$ denotes the augmented Lagrange multiplier updated as $\tu^{(t+1)} = \tu^{(t)}+\rho(\tx^{(t+1)}-\tv^{(t+1)})$. The $\tx$-subproblem is a simple quadratic optimization that admits a closed-form solution as
\begin{equation}
\tx^{(t+1)} = (\tA^{\top}\tA+\lambda\tI)^{-1}(\tA^{\top}\ty+\lambda(\tv^{(t)}-\tu^{(t)})).
\end{equation}
The intermediately reconstructed image $\tx^{(t+1)}$ depends on both the observation model and a fixed estimate of $\tv$. The $\tv$-subproblem is also called the proximity operator of $J(\tv)$ computed at point $\tx^{(t+1)}+\tu^{(t)}$, whose solution can be obtained by a denoising algorithm. By alternatively updating $\tx$ and $\tv$ until convergence, the original optimization problem of Eq. (\ref{Obj_fun}) is then solved. The advantage of this framework is that other state-of-the-art denoising algorithms, which cannot be explicitly expressed in $J(\tx)$, can also be used to update $\tv$, leading to better IR performance. For example, the well-known BM3D \cite{BM3D}, Gaussian mixture model \cite{Zoran:ICCV11}, NCSR \cite{NCSR} have been used for various IR applications \cite{PPP:13,Brifman:ICIP16,Teodoro:ICIP16}. In \cite{Zhang:CVPR17}, the sate-of-the-art CNN denoiser has also been plugged as an image prior for general IR. Due to the excellent denoising ability, state-of-the-art IR results for different IR tasks have been obtained. Similar to \cite{Bigdeli:arXiv}, an autoencoder denoiser is plugged into the objective function of Eq. (\ref{Obj_fun}). However, different from the variable splitting method described above, the objective function of \cite{Bigdeli:arXiv} is minimized by gradient descent. Though the denoising-based IR methods are very flexible and effective in exploiting sate-of-the-art image prior, they require a lot of iterations for convergence and the whole components cannot be jointly optimized.

\subsection{Deep network based IR methods}

Inspired by the great success of DCNNs for image classification \cite{Alexnet,Resnet}, object detection \cite{FRCNN:PAMI17,Liu:detection}, semantical segmentation \cite{FCN:CVPR15}, etc., DCNNs have also been applied for low-level image processing tasks \cite{SRCNN,VDSR,TNRD,Zhang:TIP17}. Similar to the coupled sparse coding \cite{Yang:SR08}, DCNNs have been proposed to learn nonlinear mapping from the LR patch space to the HR patch space \cite{SRCNN}. In \cite{Zhang:TIP17}, DCNN with residual learning has been proposed for image restoration. To improve the SR performance, very deep CNN has been developed and achieved sate-of-the-art SR results \cite{VDSR}. To alleviate the difficulty of training very deep networks, deep recursive residual learning has been proposed to train very deep networks for image SR \cite{Tai:CVPR17}. By treating deep super-resolution architecture as a single-state recurrent neural network (RNN), in \cite{Han:CVPR18} a dual-state RNN has been proposed for SR to exploit both low-resolution and high-resolution signals jointly. To reuse the feature maps from preceding layers, densely connected network has also been developed for image SR \cite{Tong:ICCV17}. Different from the existing shortcut connections for identity mappings, adaptive shortcut connections with learnable parameters have also been proposed in \cite{Zhang:TIP18} for image restoration tasks. In addition to the commonly used mean-square loss, a generative adversarial network (GAN) based SR model using perceptual loss functions has also been proposed for photo-realistic super-resolved natural images \cite{Ledig:CVPR17}. To exploit the long-term dependencies in the deep CNN, in \cite{Tai:ICCV17} very deep persistent memory network containing memory blocks has been developed, leading to substantial improvements for typical image restoration tasks. For non-blind image deblurring, multiplayer perceptron network \cite{MLP:CVPR12} has been developed to remove the deconvolution artifacts. In \cite{Xu:NIPS14}, Xu et al. propose to use DCNN for non-blind image deblurring. Though excellent IR performances have been obtained, these DCNN methods generally treat the IR problems as denoising problems, i.e., removing the noise or artifacts of the initially recovered images, and ignore the observation models.

There have been some attempts to leverage the domain knowledge and the observation model for IR. In \cite{Wang:CVPR15}, based on the learned iterative shrinkage/thresholding algorithm (LISTA) \cite{LISTA}, Wang et al. developed a deep network whose layers correspond to the steps of the sparse coding based image SR. In \cite{TNRD}, the classic iterative nonlinear reaction diffusion method is also implemented as a deep network, whose parameters were jointly trained. The DNN inspired from the ADMM-based sparse coding algorithm has also been developed for compressive sensing based MRI reconstruction \cite{ADMM-Net}. In \cite{Xin:NIPS16}, the truncated iterative hard thresholding algorithm for solving $\ell_0$-norm sparse recovery problem was implemented as a DNN. These model-based DNNs have shown significant improvements in terms of both efficiency and effectiveness over original iterative algorithms. However, the strict implementations of the conventional sparse coding based methods result in limited receipt fields of the convolutional filters and thus cannot exploit the spatial correlations of the feature maps effectively, leading to limited IR performance. In \cite{Kim:CVPR17} learned regularizer based on the DCNN has been proposed under the alternative minimization framework, showing very promising results for several IR tasks. However, there were also some hand-crafted components in the proposed framework, e.g., the gradient operators to extract gradient features and the preconditioned conjugate gradient (PCG) method used to reconstruct the image from the regularized gradients. Instead of learning regularizer in the gradient domain, DCNN-based image denoisers in pixel domain have also been learned as proximal operators of regularization used in convex energy minimization algorithms for image restoration \cite{Zhang:CVPR17,Meinhardt:ICCV17}.

\section{Proposed Denoising-based Image Restoration Algorithm}

In this section, we develop an efficient iterative algorithm for solving the denoising-based IR methods, based on which a feed-forward DNN will be proposed in the next section. Considering the denoising-based IR problem of Eq. (\ref{den-based}), we adopt the half-quadratic splitting method, by which the equally constrained optimization problem can be converted into a non-constrained optimization problem, as
\begin{equation}
(\tx,\tv)  = \argmin_{\tx, \tv} \frac{1}{2}||\ty-\tA\tx||_2^2 + \eta||\tx-\tv||_2^2 + \lambda J(\tv). \\
\end{equation}
The above optimization problem can be solved by alternatively solving two sub-problems,
\begin{equation}
\begin{split}
&\tx^{(t+1)} = \argmin_{\tx} ||\ty-\tA\tx||_2^2 + \eta||\tx-\tv^{(t)}||_2^2, \\
&\tv^{(t+1)} = \argmin_{\tv} \eta||\tx^{(t+1)}-\tv||_2^2 + \lambda J(\tv).
\end{split}
\end{equation}
The $\tx$-subproblem is a quadratic optimization problem that can be solved in closed-form, as $\tx^{(t+1)} = \tW^{-1}\tb$, where $\tW$ is a matrix related to the degradation matrix $\tA$. Generally, $\tW$ is very large, so it is impossible to compute its inverse matrix. Instead, the iterative classic conjugate gradient (CG) algorithm can be used to compute $\tx^{(t+1)}$, which requires many iterations for computing $\tx^{(t+1)}$. In this paper, instead of solving for an exact solution of the $\tx$-subproblem, we propose to compute $\tx^{(t+1)}$ with a single step of gradient descent for an inexact solution, as
\begin{equation}
\begin{split}
\tx^{(t+1)} &= \tx^{t} - \delta[ \tA^{\top}(\tA\tx^{(t)}-\ty) + \eta(\tx^{(t)}-\tv^{(t)}) ]  \\
&= \bar{\tA}\tx^{(t)} + \delta\tA^{\top}\ty + \delta\eta\tv^{(t)},
\end{split}
\end{equation}
where $\bar{\tA}=[(1-\delta\eta)\tI-\delta\tA^{\top}\tA]$ and $\delta$ is the parameter controlling the step size. By pre-computing $\bar{\tA}$, the update of $\tx^{(t)}$ can be computed very efficiently. As will be shown later, we do not need to solve the $\tx$-subproblem exactly. Updating $\tx^{(t+1)}$ once is sufficient for $\tx^{(t)}$ to converge to a local optimal solution. The $\tv$-subproblem is a proximity operator of $J(\tv)$ computed at point $\tx^{(t+1)}$, whose solution can be obtained by a denoiser, i.e., $\tv^{(t+1)}=f(\tx^{(t+1)})$, where $f(\cdot)$ denotes a denoiser. Various denoising algorithms can be used, including those cannot be explicitly expressed by the MAP estimator with $J(\tx)$. In this paper, inspired by the success of DCNN for image denoising, we choose a DCNN-based denoiser to exploit the large training dataset. However, different from existing DCNN models for IR, we consider the network that can exploit the multi-scale redundancies of natural images, as will be described in the next section. In summary, the proposed iterative algorithm for solving the denoising-based IR problem is summarized in \textbf{Algorithm 1}, where we initialize $\tx^{(0)}$ as $\tx^{(0)}=\tA^{\top}\ty$. Thus, for image denoising, $\tx^{(0)}=\ty$. For image SR, $\tx^{(0)}=\tA^{\top}\ty=\tH^{\top}\tD^{\top}\ty$, which can be obtained by first upsampling $\ty$ with zero interpolation and then filtering the upsampled image with transposed blur matrix $\tH^{\top}$. For image SR with bicubic downsampling, $\tx^{(0)}=\tA^{\top}\ty$ is implemented as image upsampling with bicubic interpolator. For image deblurring, $\tx^{(0)}=\tA^{\top}\ty=\tH^{\top}\ty$, which can be implemented as transposed convolution of $\ty$ with $\tH^{\top}$. We now discuss the convergence property of \textbf{Algorithm 1}.

\begin{algorithm}[tbh]
\caption{Denoising-based IR Algorithm}
$\bullet$ \textbf{Initialization}:

    \hspace{0.5cm} (1) Set observation matrix $\tA$, $\bar{\tA}$, $\delta>0$, $\eta>0$,  $t=0$;

   \hspace{0.5cm} (2) Initialize $\tx$ as $\tx^{(0)} = \tA^{\top}\ty$;

$\bullet$ \textbf{While} not converge \textbf{do}

    \hspace{0.5cm} (1) Compute $\tv^{(t+1)} = f(\tx^{(t)})$;

    \hspace{0.5cm} (2) Compute $\tx^{(t+1)} = \bar{\tA}\tx^{(t)} + \delta\tA^{\top}\ty + \delta\eta\tv^{(t+1)}$;

    \hspace{0.5cm} (3) $t = t + 1$.

\textbf{End while}

Output: $\tx^{(t)}$
\end{algorithm}

\begin{thm}[] Consider the energy function
$$\xi(\tx,\tv) := \frac{1}{2}\|\ty-\tA\tx\|_2^2 + \frac{\eta}{2}\|\tx-\tv\|_2^2 + \lambda J(\tv).$$
Assume that $\xi$ is lower bounded and coercive\footnote{$\xi(\tx,\tv)\to\infty$ whenever $\|(\tx,\tv)\|\to\infty$.}.
For Algorithm 1, $(\tx^{(t)}, \tv^{(t)})$ has a subsequence that converges to a stationary point of the the energy function provided that the denoiser $f(\cdot)$ satisfies the sufficient descent condition:
\begin{align}\label{fcond}
&\frac{\eta}{2}||\tx-\tv||_2^2 + \lambda J(\tv)-\frac{\eta}{2}||\tx-f(\tx)||_2^2 - \lambda J(f(\tx))\nonumber\\
&\ge c_2\|\tilde\nabla_{\tv}\xi(\tx,\tv)\|_2^2,
\end{align}
where $c_2>0$ and $\tilde\nabla_{\tv}\xi(\tx,\cdot)$ is a continuous limiting subgradient of $\xi$.
\end{thm}

\textit{Proof} See the Appendix.

Let us discuss the condition \eqref{fcond}. We list some combinations of the function $J$ and mapping $f$ 
that satisfy  \eqref{fcond}:
\begin{enumerate}
\item $J$ is $L$-Lipschitz differentiable, and $f:(\tx,\tv)\mapsto \tv-\alpha \nabla_{\tv} \xi(\tx,\tv)$ is a gradient descent map, where $\alpha \in (0,\frac{2}{\eta+L})$ if $\xi(\tx,\tv)$ is convex in $\tv$ or  $\alpha \in (0,\frac{1}{\eta+L})$ otherwise. Then, \eqref{fcond} follows from standard gradient analysis.
\item $J$ is proper and lower semi-continuous, the function $\xi'(\tu;\tx,\tv):=\frac{\mu}{2}\|\tx-\tu\|_2^2 + \lambda J(\tu)+\frac{\beta}{2}\|\tv-\tu\|_2^2$ is at least $\beta$-strongly   convex in $\tu$, and $f: (\tx,\tv) \mapsto \tv^+:=\argmin_{\tu}\xi'(\tu;\tx,\tv)$. This $f$ is known as the proximal mapping of $\frac{\mu}{2}\|\tx-\cdot\|_2+J(\cdot)$. The properties of $J$ ensures $\tv^+$ to be well defined. Then, by  convexity and optimality condition of the ``$\argmin$'' subproblem,
    \begin{align}&\frac{\mu}{2}\|\tx-\tv\|_2^2 + \lambda J(\tv)-\frac{\mu}{2}\|\tx-\tv^+\|_2^2 - \lambda J(\tv^+)\nonumber\\
    &\ge \beta\|\tv-\tv^+\|_2^2 = \frac{1}{\beta}\|\mu(\tx-\tv^+)+\lambda \tilde\nabla J(\tv^+)\|_2^2\nonumber\\
    &=\frac{1}{\beta}\|\tilde\nabla_{\tv}\xi(\tx,\tv^+)\|_2^2.\label{xvp}
    \end{align}
    This is different from \eqref{fcond} since the right-hand side uses $\tv^+$ rather than $\tv$. However, applying the right-hand side term $\|\tv-\tv^+\|_2^2$ in the proof yields $\lim_{t}\|\tv^{(t)}-\tv^{(t+1)} \|_2=0$ and thus \eqref{fcond} is satisfied asymptotically and the proof results still apply.

\item Let $\mathcal{M}$ denote a manifold of (noiseless) images and $J(\tv) := \mathrm{dist}(\tv,\mathcal{M})^2$ be a function that measures a certain kind of squared distance between $\tv$ and $\mathcal{M}$. In particular, consider the squared Euclidean distance $J(\tv)=\frac{1}{2}\|\tv - \Pi_{\mathcal{M}}(\tv)\|_2^2$, where $\Pi_{\mathcal{M}}(\tv)$ denotes  orthogonal projection of $\tv$ to $\mathcal{M}$. Then, for $f(\tx) := \argmin_{\tu}\{\frac{\mu}{2}\|\tx-\tu\|_2^2 + \frac{\lambda}{2}\|\tu - \Pi_{\mathcal{M}}(\tu)\|_2^2+\frac{\beta}{2}\|\tv-\tu\|_2^2\}$, we have $f(\tx)=\frac{1}{\lambda+\mu+\beta}(\mu\tx +\beta\tv+\lambda\Pi_{\mathcal{M}}(\mu\tx +\beta\tv)).$ Similar to the last point, we have \eqref{xvp} and thus \eqref{fcond} asymptotically.
\item For the same $\mathcal{M}$ in the last part, define $J(\tx)=\delta_{\mathcal{M}}(\tx)$, which returns 0 if $\tx\in\mathcal{M}$ and $\infty$ if $\tx\not\in\mathcal{M}$. If the manifold $\mathcal{M}$ is bounded and differentiable, then $J(\tx)$ is known as \emph{restricted prox-regular}. For $f(\tx) := \argmin_{\tu}\{\frac{\mu}{2}\|\tx-\tu\|_2^2 + \delta_{\mathcal{M}}(\tx)+\frac{\beta}{2}\|\tv-\tu\|_2^2\}$, It is discussed in \cite{Yin:18} that \eqref{xvp} holds and thus \eqref{fcond} holds in the asymptotic sense.
\end{enumerate}
In parts 2--3 above, we can remove the proximity term $\frac{\beta}{2}\|\tv-\tu\|_2^2$, which is used in defining the mapping $f$, and still ensure the same result, i.e., subsequence convergence to a stationary point. However, the  proof must be adapted to each $J(\tv)$ separately. We leave this to our future work.

In part 2, $J(\tx)$ is proper if the manifold is nonempty, and $J(\tx)$ is lower semi-continuous if the epigraph of the manifold is closed. These conditions are easy to check and often, though not always, satisfied. That said, the manifold needs to be first-order smooth and bounded, or having globally bound curvatures, in order for $\xi$ to be strongly convex.

It has been shown in \cite{KL:07} that if $\xi$ has the Kurdyka-\L ojasiewicz (KL) property, the subsequence convergence can be upgraded to the convergence of full sequence, which has been a standard argument in recent convergence analysis. As shown in \cite{Yin:13}, functions satisfying the KL property include, but not limited to, real analytic functions, semi-algebraic functions, and locally strongly convex functions. Therefore, $(\tx^{(t)}, \tv^{(t)})$ converges to a stationary point. It is possible that the stationary point $(\tx^{*},\tv^{*})$ is a saddle point rather than a local minimizer. However, it is known that first-order methods almost always avoid saddle points assuming the initial solution is randomly selected \cite{Lee:arXiv}. Therefore, converging to a saddle point is extremely unlikely.

It has been shown in \cite{Alain:14} that the denoiser autoencoder can be regarded as an approximately orthogonal projection of the noisy input $\ty$ to the manifold of noiseless images. Therefore, as shown in the above parts $2$, \textbf{Algorithm 1} with the mapping function $f(\cdot)$ defined by the DCNN denoiser in a loose sense converges to a local minimizer, based on the above analysis.

\section{Denoising Prior Driven Deep Neural Network}

\begin{figure*}[tbh]
	\centering
\subfigure[]{
    \includegraphics[width=0.90\linewidth]{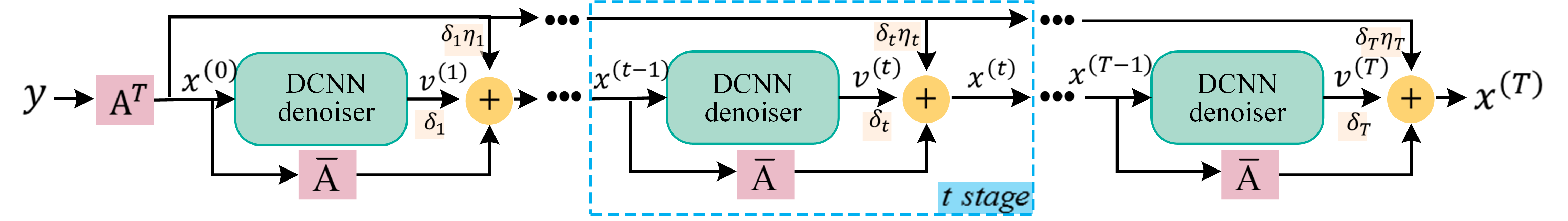}
    }
\subfigure[]{
    \includegraphics[width=0.70\linewidth]{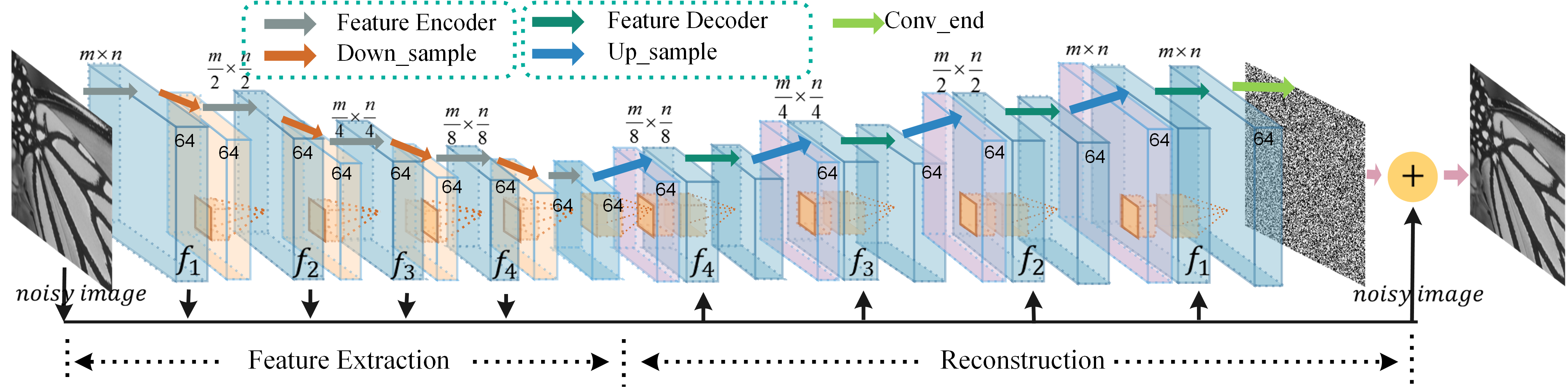}
    }
\subfigure[]{
    \includegraphics[width=0.16\linewidth]{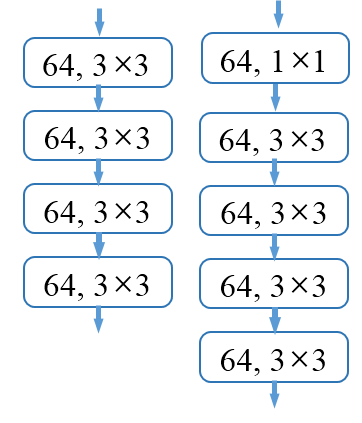}

   }\\
  \caption{Architectures of the proposed deep network for image restoration. (a) The overall architecture of the proposed deep neural network; (b) the architecture of the plugged DCNN-based denoiser; (c) the architecture of the feature extraction(left) and the reconstruction(right)}.
  \label{fig:network}

\end{figure*}

In general, \textbf{Algorithm 1} requires many iterations to converge and is computationally expensive. Moreover, the parameters and the denoiser cannot be jointly optimized in an end-to-end training manner. To address these issue, here we propose to unfold the \textbf{Algorithm 1} into a deep network of the architecture shown in Fig. \ref{fig:network} (a). The network exactly executes $T$ iterations of \textbf{Algorithm 1}. The input degraded image $\ty\in \mathbb{R}^{n_y}$ first goes through a linear layer parameterized by the degradation matrix $\tA\in \mathbb{R}^{n_y\times m_x}$ for an initial estimate $\tx^{(0)}$. $\tx^{(0)}$ is then fed into the denoiseing module and the linear layer parameterized by matrix $\bar{\tA}\in \mathbb{R}^{m_x\times m_x}$. The denoised signal $\tv^{(1)}$ weighted by $\delta_{1,1}$ is then added with the output of the linear layer $\bar{\tA}$ and $\tA^{\top}\ty$ weighted by $\delta_{1,2}$ via a shortcut connection to obtain the updated $\tx^{(1)}$. The structure of the denoising module is shown in Fig. \ref{fig:network}(b). Such a process is repeated $T$ times. In our implementation, $T=6$ was always used. Instead of using fixed weights, all the weights $\delta_{t,1}$, $\delta_{t,2}$, $t=1,2,\cdots, T$ involved in the $T$ recurrent stages can be discriminatively learned through end-to-end training. Regarding the denoising module, as we are using a DCNN-based denoiser that contains a large number of parameters, we enforce all the denoising modules to share the same parameters to avoid over-fitting.

The linear layers $\tA^{\top}$ and $\bar{\tA}$ are also trainable for a typical degradation matrix $\tA$. For image denoising, $\tA=\tA^{\top}=\tI$, and $\bar{\tA}$ also reduces to a weighted identity matrix $\bar{\tA}=\lambda\tI$, where $\lambda=1-\delta(1+\eta)$. For image deblurring, the layer $\tA^{\top}$ can be simply implemented with a convolutional layer. The layer $\bar{\tA}=a\tI-\delta\tA^{\top}\tA$ can also be computed efficiently by convolutional operations. The weight $a$ and filters correspond to $\tA^{\top}$ and $\tA$ can also be discriminatively learned. For image super-resolution, two types of degradation operators are considered: the Gaussian downsampling and the bicubic downsampling. For Gaussian downsampling, $\tA=\tD\tH$, where $\tH$ and $\tD$ denote the Gaussian blur matrix and the downsampling matrix, respectively. In this case, the layer $\tA^{\top}=\tH^{\top}\tD^{\top}$ corresponds to first upsample the input LR image by zero-padding and then convolute the upsampled image with a filter. Layer $\bar{\tA}$ can also be efficiently computed with convolution, downsampling and upsampling operations. All convolutional filters involved in these operations can be discriminatively learned. For bicubic downsampling, we simply use the bicubic interpolator function with scaling factor $s$ and $1/s$ ($s=2,3,4$) to implement the matrix-vector multiplications $\tA^{\top}\ty$ and $\tA\tx$, respectively.

\subsection{The DCNN denoiser}

Inspired by the recent advances on semantical segmentation \cite{FCN:CVPR15} and object segmentation \cite{Sharpmask}, the architecture of the denoising network is illustrated in Fig. \ref{fig:network}(b). Note that other more powerful denoising network can also be used in the proposed IR framework. Similar to the U-net \cite{UNet} and the sharpMask net \cite{Sharpmask}, the denoising network contains two parts: the feature extraction and image reconstruction parts. In the feature extraction part, there are a series of convolutional layers followed by downsampling layers to reduce the spatial resolution of the feature maps. The downsampling layer helps increasing the receipt field of the neurons. The convolutional layers are grouped into $L$ feature encoding blocks ($L=6$ in our implementation), as shown by the gray arrows in Fig. \ref{fig:network}(b). As shown in Fig. \ref{fig:network}(c), each feature encoding block contains four convolutional layers with ReLU nonlinearity and $3\times 3$ kernels, each of which generates 64-channel feature maps. The first four encoding blocks are followed by a downsampling layer to reduce the spatial resolution of the feature maps with scaling factor 0.5. In the downsampling layers, the feature maps are sub-sampled by scaling factor 2 along both axes.

The image reconstruction part also contains a series of convolutional layers, which are grouped into four feature decoding blocks (as shown by the green arrows in Fig. \ref{fig:network}(b)) followed by upsampling layers to increase the spatial resolution of the feature maps. As the finally extracted feature maps lose a lot of spatial information, directly reconstructing images from the extracted features cannot recover fine image details. To compensate the lost spatial information, the feature maps of the same spatial resolution generated in the encoding stage are fused with the upsampled feature maps generated in the decoding stage, for obtaining newly upsampled feature maps. As shown in Fig. \ref{fig:network}(c), each decoding block consists of five convolutional layers. The first layer reduces the number of feature maps from $128$ to $64$ with $1\times 1$ kernels and ReLU function. The following four layers generates $64$-channel feature maps with $3\times 3$ kernels with ReLU nonlinearity. The generated feature maps of the last layer are then upsampled with scaling factor $2$ by a deconvolution layer. The upsampled feature maps are then fused with the feature maps of the same spatial resolution from the encoding part. Specifically, the fusion is conducted by concatenating the feature maps. The output image is reconstructed from the $64$-channel feature maps with a filter of size $3\times 3$. Instead of reconstructing the original image directly, we enforce the denoising network to predict the residual, which has been verified to be more robust \cite{Zhang:TIP17}. To this end, a skip connection from the input to the reconstructed image was added.

\subsection{Overall network training}

Note that the DCNN denoisers do not have to be pre-trained. Instead, the overall deep network shown in Fig. \ref{fig:network} (a) is trained by end-to-end training. To reduce the number of parameters and thus avoid over-fitting, we enforce each DCNN denoiser to share the same parameters. Mean square error (MSE) based loss function is adopt to train the proposed deep network, which can be expressed as
\begin{equation}
\tTheta = \argmin_{\tTheta} \sum_{i=1}^N ||\mathcal{F}(\ty_i; \tTheta) - \tx_i||_2^2,
\end{equation}
where $\ty_i$ and $\tx_i$ denote the $i$-th pair of degraded and original image patches, respectively, and $\mathcal{F}(\ty_i; \tTheta)$ denotes the reconstructed image patch by the network with parameter set $\tTheta$. It is also possible to train the network with other the perceptual based loss functions, which may lead to better visual quality. We remain this as future work. The ADAM optimizer \cite{ADAM} is used to train the network with setting $\beta_1=0.9$, $\beta_2=0.999$ and $\epsilon=10^{-8}$. The convolutional kernels were initialized by Xavier initializers developed in \cite{He:ICCV15}. The linear layers related to the degradation matrix $\tA$ were initialized by the degradation model $\tA$. Other parameters, i.e., $\delta$ and $\eta$ were empirically initialized as $0.1$ and $0.9$, respectively. The proposed networks were trained with a minibatch size of $16$. The learning rate $\epsilon$ was initialized as $0.0005$ and halved at every $43000$ minibatch updates. The proposed network is implemented under the Tensorflow framework and trained using $4$ Nvidia Titan Xp GPUs, taking about one day to converge. Note that since all the layers of the proposed network are convolutional, the input degraded image can be of arbitrary sizes.

\section{Experimental Results}

In this section, we perform several IR tasks to verify the performance of the proposed network, including image denoising, deblurring, and super-resolution. We trained each model for different IR tasks. We empirically found that implementing $T=6$ iterations of \textbf{Algorithm 1} in the network generally lead to satisfied IR results for image denoising, deblurring and super-resolution tasks. Thus, we fixed $T=6$ for all IR tasks. To train the networks, we constructed a large training image set, consisting of $1000$ images of size $256\times 256$ used in \cite{Dong:TIP11}.

\subsection{Ablation study}

To show the effects of the initialization of the degradation matrix $\tA$, we implemented the proposed network using two types of initializations, i.e., initializing the linear layers related to $\tA$ using the degradation matrix $\tA$ (denoted as DPDNN-A) and random initialization (denoted as DPDNN-Random). Table \ref{table:A_SR_deblur} shows this comparison study. For image super-resolution, as we implement the degradation matrix $\tA$ using the bicubic interpolator function for SR with bicubic downsampling, we only show the ablation study for image SR with Gaussian downsampling. From Table \ref{table:A_SR_deblur}, one can see that the two initializations lead to similar results for both image deblurring and SR tasks, indicating that the network can learn the linear layers related to degradation matrix from scratches.

\begin{table}[H]
	\begin{center}	
		\caption{Ablation study on the effects of the initialization of the layers related to degradation matrix. Average PSNR results of image deblurring and super-resolution by the proposed networks.}
		\setlength{\tabcolsep}{4.5pt}
		\label{table:A_SR_deblur}
		\begin{tabular}{|p{2.02cm}<{\centering}|p{0.5cm}<{\centering}|p{0.5cm}<{\centering}|p{0.5cm}<{\centering}|p{0.5cm}<{\centering}|p{0.5cm}<{\centering}|p{0.5cm}<{\centering}|p{0.5cm}<{\centering}|p{0.5cm}<{\centering}|}
			\hline
			Task         & \multicolumn{5}{c|}{Image debluring}                                     & \multicolumn{2}{c|}{SISR}                  \\ \hline
			Data set     & \multicolumn{5}{c|}{Set10}                                               & Set5                       & Set14                     \\ \hline
			kernel       & \multicolumn{2}{c|}{Kernel 1} & \multicolumn{2}{c|}{Kernel 2} & Gaussian & \multicolumn{2}{c|}{\multirow{2}{*}{ \begin{tabular}[c]{@{}c@{}}Scaling\\ factor 3\end{tabular}}} \\ \cline{1-6}
			Noise level  & 2.55          & 7.65          & 2.55          & 7.65          & 2.0      & \multicolumn{2}{c|}{}                                  \\ \hline
			DPDNN-Random & 33.19         & 29.01         & 32.64         & 28.54         & 30.79    & 34.22                      & 29.88                     \\ \hline
			DPDNN-A      & 33.24         & 29.09         & 32.66         & 28.58         & 30.74    & 34.20                      & 29.91                     \\ \hline
		\end{tabular}
	\end{center}
\end{table}

We have also conducted ablation study on the effects of the initialization of the denoiser, i.e., implementing the proposed network with pre-trained denoiser (denoted as DPDNN-Pretrain) or randomly initialized denoiser (denoted as DPDNN-Random). Total $450,000$ patches of size $40\times 40$ were extracted for training. The noise levels in the range of $[0, 50]$ were used to simulate the noisy image patches. Tables \ref{table:Pre_ran_denoise}$\sim$\ref{table:Pre_ran_deblur} show the average PSNR results for image denoising, SR and deblurring tasks by the proposed method, respectively. From Tables \ref{table:Pre_ran_denoise}$\sim$\ref{table:Pre_ran_deblur}, we can see that the two initializations of the denoiser also lead to similar results. The above two ablation studies show that the proposed network is insensitive to the initialization of the parameters. The reason is that the number of network parameters is controllable, as we enforce each denoiser to share the same parameters, and thus the network can be effectively learned from scratches.

\begin{table}[b]
	\begin{center}
		
		\caption{Ablation study on the effects of the initializations of the denoiser for image denoising. Average PSNR results on Set12 and BSD68 datasets.}
		\setlength{\tabcolsep}{4.5pt}
		\label{table:Pre_ran_denoise}
		\begin{tabular}{|c|c|c|c|c|}
			\hline
			Dataset & \multicolumn{2}{c|}{Set12}    & \multicolumn{2}{c|}{BSD68}    \\ \hline
			$\delta$   & \begin{tabular}[c]{@{}c@{}}DPDNN-\\ Random\end{tabular}  & \begin{tabular}[c]{@{}c@{}}DPDNN-\\ Pretrain\end{tabular} & \begin{tabular}[c]{@{}c@{}}DPDNN-\\ Random\end{tabular} &  \begin{tabular}[c]{@{}c@{}}DPDNN-\\ Pretrain\end{tabular} \\ \hline
			15      & 32.91        & 32.86          & 32.29        & 32.27          \\ \hline
			25      & 30.54        & 30.50          & 29.88        & 29.86          \\ \hline
			50      & 27.50        & 27.53          & 27.02        & 27.03          \\ \hline
		\end{tabular}
	\end{center}
\end{table}

\begin{table}[b]
	\begin{center}
		
		\caption{Ablation study on the effects of the initializations of the denoiser for image super-resolution. Average PSNR results on Set14 and BSD100 datasets.}
		\setlength{\tabcolsep}{3.5pt}
		\label{table:Pre_ran_SR}
		\begin{tabular}{|p{1.0cm}<{\centering}|p{0.15cm}<{\centering}|p{0.85cm}<{\centering}|p{0.9cm}<{\centering}|p{0.9cm}<{\centering}|p{0.9cm}<{\centering}|p{0.9cm}<{\centering}|p{0.85cm}<{\centering}|}
			\hline
			\multicolumn{2}{|c|}{Dataset}   & \multicolumn{2}{c|}{Set14}    & \multicolumn{2}{c|}{BSD100}   & \multicolumn{2}{c|}{Urban100} \\ \hline
			\multicolumn{2}{|c|}{Methods}   &\begin{tabular}[c]{@{}c@{}}DPDNN\\ Random\end{tabular}  & \begin{tabular}[c]{@{}c@{}}DPDNN\\ Pretrain\end{tabular} & \begin{tabular}[c]{@{}c@{}}DPDNN\\ Random\end{tabular}  & \begin{tabular}[c]{@{}c@{}}DPDNN\\ Pretrain\end{tabular} & \begin{tabular}[c]{@{}c@{}}DPDNN\\ Random\end{tabular}  & \begin{tabular}[c]{@{}c@{}}DPDNN\\ Pretrain\end{tabular} \\ \hline
			\multirow{3}{*}{Bicubic} & 2 & 33.30        & 33.04          & 32.09        & 32.04          & 31.50        &   31.49             \\ \cline{2-8}
			& 3 & 30.02        & 30.01          & 29.00        & 28.91          & 27.61        &   27.59             \\ \cline{2-8}
			& 4 & 28.28        & 28.29          & 27.44        & 27.39          & 25.53        &   25.55             \\ \hline
			{Gaussian}& 3 & 29.63        &   29.61   &     28.89   &   28.85         &   26.12      &   26.13       \\ \hline
		\end{tabular}
	\end{center}
\end{table}

\begin{table}[t]
	\begin{center}
		\caption{Ablation study on the effects of the initializations of the denoiser for image deblurring. Average PSNR results on Set10 datasets.}
		\setlength{\tabcolsep}{3.5pt}
		\label{table:Pre_ran_deblur}
	\begin{tabular}{|c|c|c|c|c|c|}
		\hline
		kernel         & \multicolumn{2}{c|}{Kernel1} & \multicolumn{2}{c|}{Kernel2} & Gaussian \\ \hline
		$\delta$          & 2.55         & 7.65         & 2.55          & 7.65         & 2        \\ \hline
		DPDNN-Random   & 33.19        & 29.01        & 32.64         & 28.54        & 30.79    \\ \hline
		DPDNN-Pretrain & 33.17        & 29.09        & 32.68         & 28.59        & 30.89    \\ \hline
	\end{tabular}
	\end{center}
\end{table}

To show the effects of the incorporation of the degradation model, we also trained the denoising network for image deblurring and SR. The structure of the denoising network is shown in Fig.\ref{fig:network}(b). We compared the denoising network (denoted as Den-network) to the proposed DPDNN for image deblurring and SR. The comparison results are shown in Tables \ref{table:deblur_DPDNN}$\sim$\ref{table:SR_DPDNN}. From Tables \ref{table:deblur_DPDNN}$\sim$\ref{table:SR_DPDNN}, we can see that the proposed DPDNN method performs much better than the denoising network. For image deblurring and SR, the average PSNR gains over the denoising network can be up to 0.42 dB and 0.63 dB, respectively, demonstrating the advantages of incorporating the degradation model into the network.

	\begin{table}[H]
	\begin{center}
		
		\caption{Average PSNR results of deblurred images on Set10 dataset by the denoising network and the proposed DPDNN.}
		\setlength{\tabcolsep}{4.5pt}
		\label{table:deblur_DPDNN}
		\begin{tabular}{|c|c|c|c|c|c|}
			\hline
			kernel     & \multicolumn{2}{c|}{Kernel1} & \multicolumn{2}{c|}{Kernel2} & Gaussian \\ \hline
			$\delta$       & 2.55          & 7.65         & 2.55          & 7.65         & 2        \\ \hline
			Den-network & 33.04         & 28.85        & 32.22         & 28.20        & 30.57    \\ \hline
			DPDNN       & 33.19         & 29.01        & 32.64         & 28.54        & 30.79    \\ \hline
		\end{tabular}
	\end{center}
\end{table}

\begin{table}[H]
	\begin{center}
		\caption{Average PSNR results of reconstructed HR images by the denoising network and the proposed DPDNN.}
		\setlength{\tabcolsep}{4.5pt}
		\label{table:SR_DPDNN}
		\begin{tabular}{|p{1.0cm}<{\centering}|p{0.12cm}<{\centering}|p{0.77cm}<{\centering}|p{0.78cm}<{\centering}|p{0.77cm}<{\centering}|p{0.78cm}<{\centering}|p{0.77cm}<{\centering}|p{0.78cm}<{\centering}|}
			\hline
			\multicolumn{2}{|c|}{Dataset}  & \multicolumn{2}{c|}{Set14}                                     & \multicolumn{2}{c|}{BSD100}                                    & \multicolumn{2}{c|}{Urban100}                                  \\ \hline
			\multicolumn{2}{|c|}{Method}    & \begin{tabular}[c]{@{}c@{}}Den-\\ network\end{tabular} & DPDNN & \begin{tabular}[c]{@{}c@{}}Den-\\ network\end{tabular} & DPDNN & \begin{tabular}[c]{@{}c@{}}Den-\\ network\end{tabular} & DPDNN \\ \hline
			\multirow{3}{*}{Bicubic} & 2 & 33.12                                                 & 33.30 & 31.98                                                  & 32.10 & 31.21                                                      & 31.50 \\ \cline{2-8}
			& 3 & 29.39                                                  & 30.02 & 28.46                                                  & 29.00 & 26.98                                                  & 27.61 \\ \cline{2-8}
			& 4 & 27.80                                                  & 28.28 & 27.19                                                  & 27.43 & 25.02                                                  & 25.53 \\ \hline
			Gaussian                  & 3 & 29.49                                                  & 29.88 & 28.51                                                     & 28.89    &  25.96                                                    & 26.12    \\ \hline
		\end{tabular}
	\end{center}
\end{table}

\subsection{Image denoising}

For image denoising, $\tA=\tI$ and \textbf{Algorithm 1} reduce to the iterative denoising process, i.e., the weighted noise image is added back to the denoised image for the next denoising process. Such iterative denoising has shown improvements over conventional denoising methods that only denoise once \cite{CSR}. Here, we also found that implementing multiple denoising iterations in the proposed network improves the denoising results. To train the network, we extracted image patches of size $40\times 40$ from the training images and added additive Gaussian noise to the extracted patches to generate the noisy patches. Totally $N=450,000$ patches were extracted for training. Note that none of the test images was included into the training image set. The training patches were also augmented by flip and rotations. We compared the proposed network with several leading denoising methods, including three model-based denoising methods, i.e., BM3D method \cite{BM3D}, the EPLL method \cite{Zoran:ICCV11}, and the low-rank based method WNNM method \cite{WNNM}, and three deep learning based methods, i.e., the TNRD method \cite{TNRD}, the DnCNN-S method \cite{Zhang:TIP17} and the MemNet \cite{Tai:ICCV17}.

Table \ref{table:denoising} shows the PSNR results of the competing methods on a set of commonly used test images shown in Fig. \ref{fig:den0}. It can be seen that the MemNet method performs comparable with the DnCNN-S method for low noise levels and outperforms DnCNN-S for higher noise level. The proposed network slightly outperforms the MemNet method by up to $0.2$ dB on average. To further verify the effectiveness of the proposed method, we also employ the Berkeley segmentation dataset (BSD68) that contains 68 natural images for comparison study. Table \ref{table:denoising2} shows the average PSNR and SSIM results of the test methods on BSD68. One can seen that the PSNR gains over the other test methods become even larger for higher noise levels. The proposed method outperforms the MemNet method by up to $0.65$ dB on average on the BSD68, demonstrating the effectiveness of the proposed method. Parts of the denoised images by the test methods are shown in Figs. \ref{fig:den1}-\ref{fig:den2}. One can see that the image edges and textures recovered by model-based methods, i.e., BM3D, WNNM and EPLL are over-smoothed. The deep learning based methods, TNRD, DnCNN-S, MemNet and the proposed method produce much more visually pleasant image structures. Moreover, the proposed method generates even better results in recovering image details than TNRD, DnCNN-S and MemNet methods.

\begin{figure*}[!tbh]
\renewcommand{\arraystretch}{0.4}
\centering
\subfigure[C.Man]{
    \includegraphics[width=0.06\textwidth]{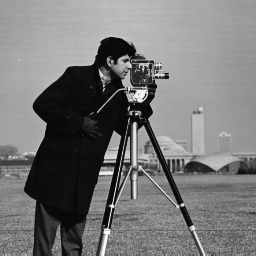}
    }
\subfigure[House]{
    \includegraphics[width=0.06\textwidth]{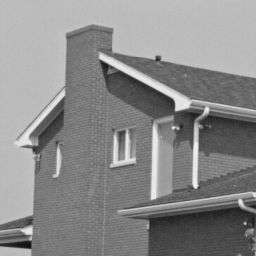}
    }
\subfigure[Peppers]{
    \includegraphics[width=0.06\textwidth]{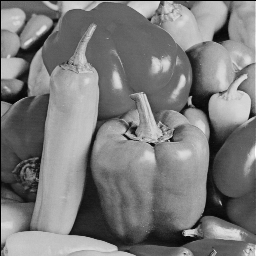}
    }
\subfigure[Starfish]{
    \includegraphics[width=0.06\textwidth]{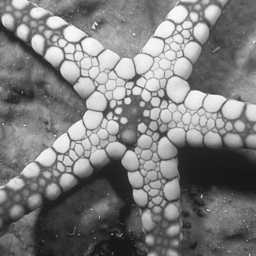}
    }
\subfigure[Monar.]{
    \includegraphics[width=0.06\textwidth]{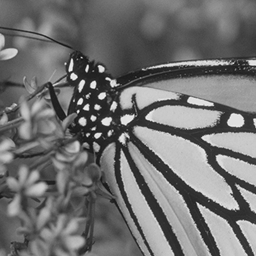}
    }
\subfigure[Airpl.]{
    \includegraphics[width=0.06\textwidth]{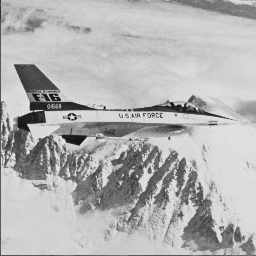}
    }
\subfigure[Parrot]{
    \includegraphics[width=0.06\textwidth]{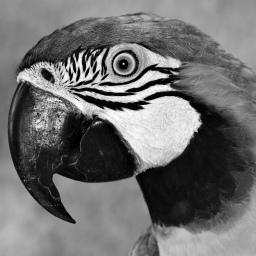}
    }
\subfigure[Lena]{
    \includegraphics[width=0.06\textwidth]{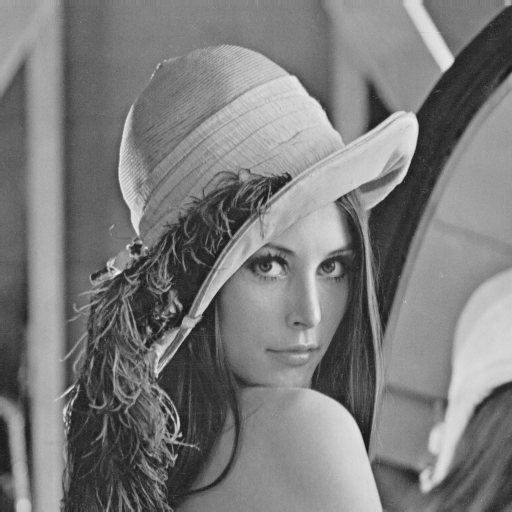}
    }
\subfigure[Barbara]{
    \includegraphics[width=0.06\textwidth]{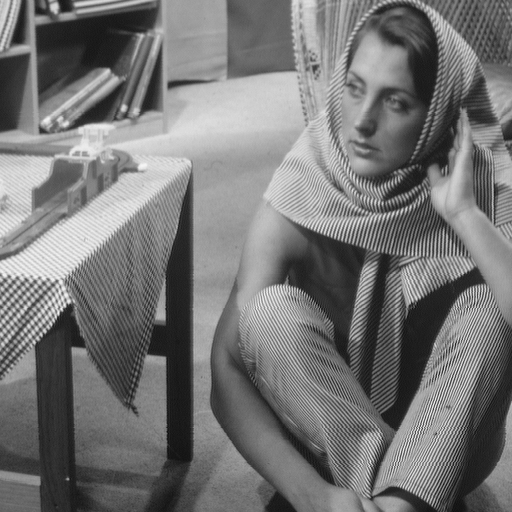}
    }
\subfigure[Boat]{
    \includegraphics[width=0.06\textwidth]{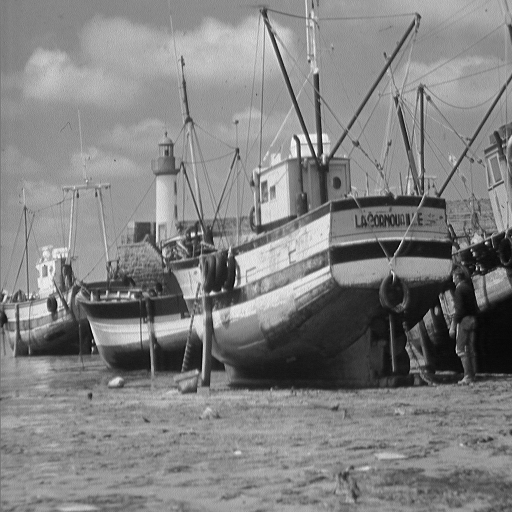}
    }
\subfigure[Man]{
    \includegraphics[width=0.06\textwidth]{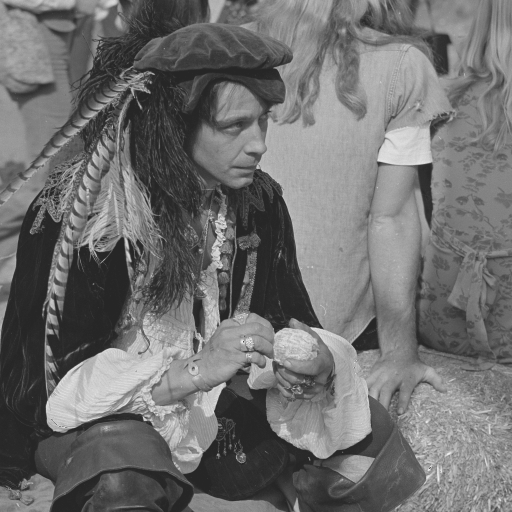}
    }
\subfigure[Couple]{
    \includegraphics[width=0.06\textwidth]{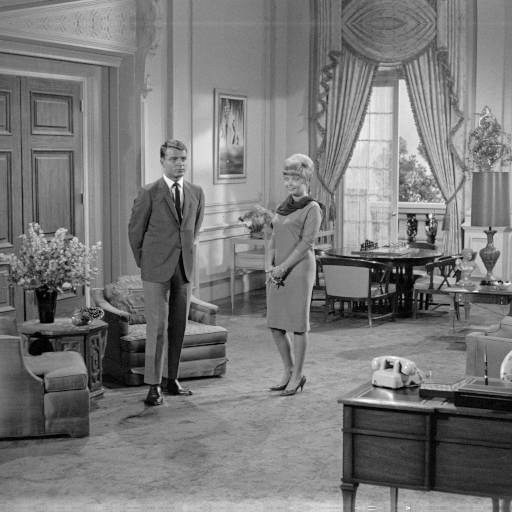}
    }
    \\
     \caption{The test images used for image denoising.}
     \label{fig:den0}
\end{figure*}

\begin{table*}[tbh]
\centering
\caption{The PSNR (dB) results of the denoised images by the test methods on a set of test images.}
\setlength{\tabcolsep}{4.5pt}
\label{table:denoising}
	\begin{tabular}{!{\vrule width1.2pt}c!{\vrule width1.2pt}c|c|c|c|c|c|c|c|c|c|c|c!{\vrule width1.2pt}c!{\vrule width1.2pt}}
       \Xhline{1.2pt}
		IMAGE       & C.Man          & House          & Peppers        & Starfish       & Monar          & Airpl          & Parrot         & Lena           & Barbara        & Boat           & Man            & Couple         & Average        \\ \Xhline{1.2pt}
		Noise Level & \multicolumn{13}{c|}{$\sigma=15$}                                                                                                                                                                                      \\ \Xhline{1.2pt}
		BM3D        & 31.92          & 34.94          & 32.70          & 31.15          & 31.86          & 31.08          & 31.38          & 34.27          & 33.11          & 32.14          & 31.93          & 32.11          & 32.38          \\ \hline
		WNNM        & 32.18          & 35.15          & 32.97          & 31.83          & 32.72          & 31.40          & 31.61          & 34.38          & \textbf{33.61} & 32.28          & 32.12          & 32.18          & 32.70          \\ \hline
		EPLL        & 31.82          & 34.14          & 32.58          & 31.08          & 32.03          & 31.16          & 31.40          & 33.87          & 31.34          & 31.91          & 31.97          & 31.90          & 32.10          \\ \hline
		TNRD        & 32.19          & 34.55          & 33.03          & 31.76          & 32.57          & 31.47          & 31.63          & 34.25          & 32.14          & 32.15          & 32.24          & 32.11          & 32.51          \\ \hline
		DnCNN-S     & \textbf{32.62} & 35.00          & \textbf{33.29} & \textbf{32.23} & 33.10          & 31.70          & \textbf{31.84} & 34.63          & 32.65          & 32.42          & 32.47          & 32.47          & 32.87          \\ \hline
		MemNet      &  32.51         &  35.10         & 33.31          & 32.12          & 33.04          & 31.53          & 31.73          & 34.64          & 32.65          & 32.43          & 32.45          & 32.49          & 32.83          \\ \hline
		Ours        & 32.44          & \textbf{35.40} & 33.19          & 32.08          & \textbf{33.33} & \textbf{31.78} & 31.48          & \textbf{34.80} & 32.84          & \textbf{32.55} & \textbf{32.53} & \textbf{32.51} & \textbf{32.91} \\ \Xhline{1.2pt}
		Noise Level & \multicolumn{13}{c|}{$\sigma=25$}                                                                                                                                                                                      \\ \Xhline{1.2pt}
		BM3D        & 29.45          & 32.86          & 30.16          & 28.56          & 29.25          & 28.43          & 28.93          & 32.08          & 30.72          & 29.91          & 29.62          & 29.72          & 29.98          \\ \hline
		WNNM        & 29.64          & 33.23          & 30.40          & 29.03          & 29.85          & 29.69          & 29.12          & 32.24          & \textbf{31.24} & 30.03          & 29.77          & 29.82          & 30.26          \\ \hline
		EPLL        & 29.24          & 32.04          & 30.07          & 28.43          & 29.30          & 28.56          & 28.91          & 31.62          & 28.55          & 29.69          & 29.63          & 29.48          & 29.63          \\ \hline
		TNRD        & 29.71          & 32.54          & 30.55          & 29.02          & 29.86          & 28.89          & 29.18          & 32.00          & 29.41          & 29.92          & 29.88          & 29.71          & 30.66          \\ \hline
		DnCNN-S     & \textbf{30.19} & 33.09          & 30.85          & 29.40          & 30.23          & 29.13          & \textbf{29.42} & 32.45          & 30.01          & 30.22          & 30.11          & 30.12          & 30.43          \\ \hline
		MemNet      & 30.02          & 33.25          & 30.87          & 29.35          & 30.24          & 29.03          & 29.30          & 32.51          & 29.98          & 30.21          & 30.08          & 30.14          & 30.41          \\ \hline
		Ours        & 30.12          & \textbf{33.54} & \textbf{30.90} & \textbf{29.43} & \textbf{30.31} & \textbf{29.14} & 29.28          & \textbf{32.69} & 30.30          & \textbf{30.34} & \textbf{30.15} & \textbf{30.24} & \textbf{30.54} \\ \Xhline{1.2pt}
		Noise Level & \multicolumn{13}{c|}{$\sigma=50$}                                                                                                                                                                                      \\ \Xhline{1.2pt}
		BM3D        & 26.13          & 29.69          & 26.68          & 25.04          & 25.82          & 25.10          & 25.90          & 29.05          & 27.23          & 26.78          & 26.81          & 26.46          & 26.73          \\ \hline
		WNNM        & 26.42          & 30.33          & 26.91          & 25.43          & 26.32          & 25.42          & 26.09          & 29.25          & \textbf{27.79} & 26.97          & 26.94          & 26.64          & 27.04          \\ \hline
		EPLL        & 26.02          & 28.76          & 26.63          & 25.04          & 25.78          & 25.24          & 25.84          & 28.43          & 24.82          & 26.65          & 26.72          & 26.24          & 26.35          \\ \hline
		TNRD        & 26.62          & 29.48          & 27.10          & 25.42          & 26.31          & 25.59          & 26.16          & 28.93          & 25.70          & 26.94          & 26.98          & 26.50          & 26.81          \\ \hline
		DnCNN-S     & 27.00          & 30.02          & 27.29          & 25.70          & 26.77          & 25.87          & \textbf{26.48} & 29.37          & 26.23          & 27.19          & 27.24          & 26.90          & 27.14          \\ \hline
		MemNet      & 27.24          & 30.70          & 27.51          & 25.76          & 27.19          & 25.96          & 26.50          & 29.63          & 26.68          & 27.30          & 27.24      & 27.14          & 27.40          \\ \hline
		Ours        & \textbf{27.12} & \textbf{31.04} & \textbf{27.44} & \textbf{25.95} & \textbf{27.00} & \textbf{25.97} & 26.42          & \textbf{29.85} & 27.21          & \textbf{27.42} & \textbf{27.32} & \textbf{27.23} & \textbf{27.50} \\ \Xhline{1.2pt}
	\end{tabular}
\end{table*}

\begin{table*}[tbh]
\centering
\caption{The average PSNR (dB) results of the competing methods on BSD68 image set. }
\setlength{\tabcolsep}{4.5pt}
\label{table:denoising2}
\begin{tabular}{!{\vrule width1.2pt}c!{\vrule width1.2pt}c!{\vrule width1.2pt}c
|c|c|c|c|c|c|c|c|c|c|c!{\vrule width1.2pt}}
\Xhline{1.2pt}
\multirow{2}{*}{Dataset}  & \multirow{2}{*}{$\sigma$} & \multicolumn{2}{c|}{BM3D} & \multicolumn{2}{c|}{EPLL} & \multicolumn{2}{c|}{TNRD} & \multicolumn{2}{c|}{DnCNN-S}   & \multicolumn{2}{c|}{MemNet} & \multicolumn{2}{c!{\vrule width1.2pt}}{Ours} \\ \Xcline{3-14}{1.2pt}
                          &                                 & PSNR        & SSIM        & PSNR         & SSIM        & PSNR        & SSIM        & PSNR         & SSIM        & PSNR         & SSIM     & PSNR        & SSIM        \\ \Xhline{1.2pt}
\multirow{3}{*}{BSD68}    & 15                      & 31.08            &     0.872 &           31.19 	 &     0.883       &     31.42         &       0.883       &    31.74    &    \textbf{0.891}   & 32.14  & 0.884 &     \textbf{32.29}  &    0.888     \\ \cline{2-14}
                         & 25                      & 28.57	           &     0.802         &            28.68 &      0.812       &     28.91        &      0.816         &    29.23  	   &    \textbf{ 0.828  }  & 29.77  & 0.819  &      \textbf{29.88} &     0.827    \\ \cline{2-14}
                          & 50                      & 25.62	           &     0.687     &          25.68   &      0.688       &      25.96	       &       0.702       &   26.24     &     0.719    & 26.37 & \textbf{0.729} &    \textbf{27.02}     &  0.726 \\ \Xhline{1.2pt}
\end{tabular}
\end{table*}

\begin{figure*}[!tbh]
\renewcommand{\arraystretch}{0.4}
\centering
\subfigure[Original]{
    \includegraphics[width=0.16\textwidth]{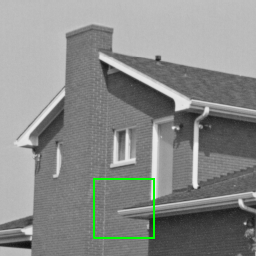}
    }\hspace{-0.8em}

\subfigure[WNNM]{
    \includegraphics[width=0.16\textwidth]{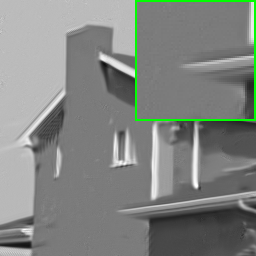}
    }\hspace{-0.8em}
\subfigure[TNRD]{
    \includegraphics[width=0.16\textwidth]{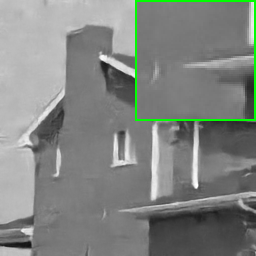}
    }\hspace{-0.8em}
\subfigure[DnCNN-S]{
    \includegraphics[width=0.16\textwidth]{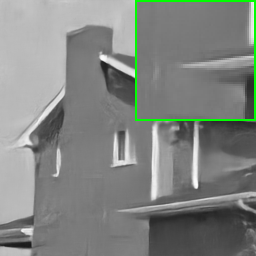}
    }\hspace{-0.8em}
\subfigure[MemNet]{
	\includegraphics[width=0.16\textwidth]{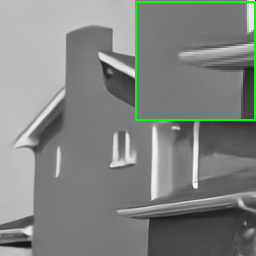}
}\hspace{-0.8em}
\subfigure[Ours]{
    \includegraphics[width=0.16\textwidth]{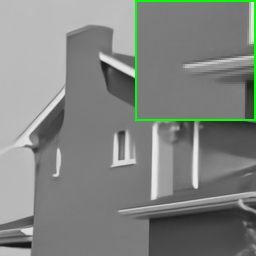}
    }
    \\
        \caption{Denoising results for \emph{House} image with noise level 50. (a) Original image; images denoised by  (b) WNNM \cite{WNNM} (30.33 dB); (c) TNRD \cite{TNRD} (29.48 dB); (d) DnCNN-S\cite{Zhang:TIP17} (30.02 dB);(e) MemNet \cite{Tai:ICCV17} (30.70 dB); (f) Ours (\textbf{31.04} dB).}
    \label{fig:den1}

\end{figure*}

\begin{figure*}[!tbh]
\renewcommand{\arraystretch}{0.4}
\centering
\subfigure[Original]{
    \includegraphics[width=0.16\textwidth]{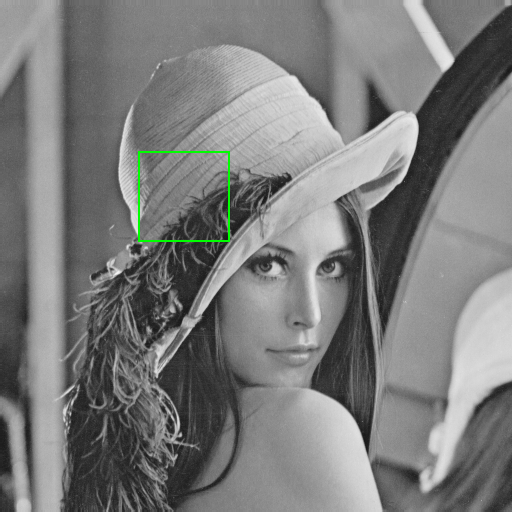}
    }\hspace{-0.8em}

\subfigure[WNNM]{
    \includegraphics[width=0.16\textwidth]{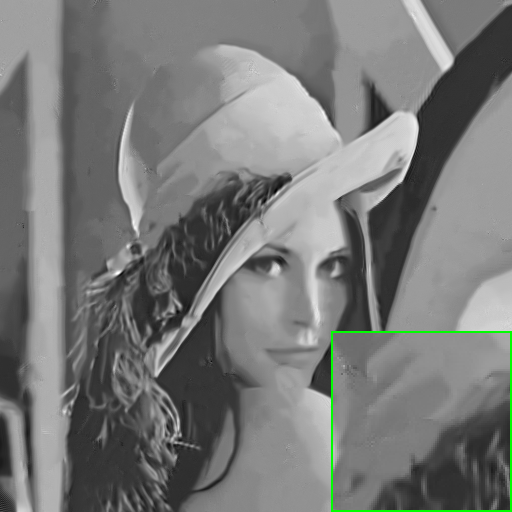}
    }\hspace{-0.8em}
\subfigure[TNRD]{
    \includegraphics[width=0.16\textwidth]{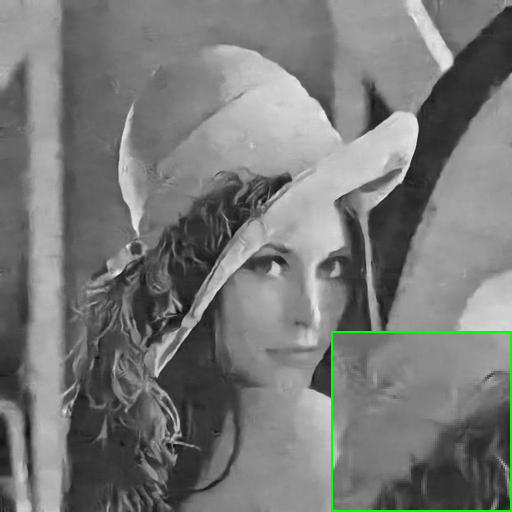}
    }\hspace{-0.8em}
\subfigure[DnCNN-S]{
    \includegraphics[width=0.16\textwidth]{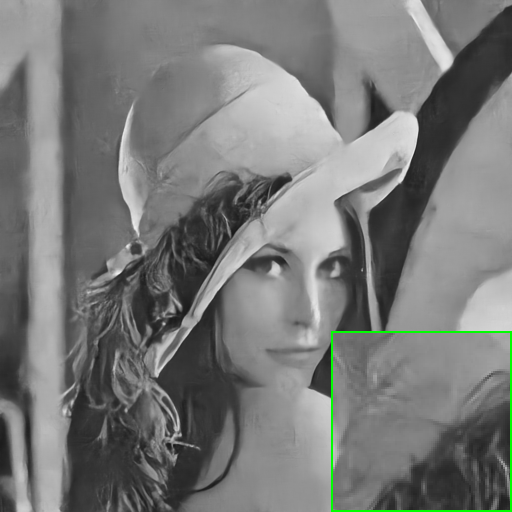}
    }\hspace{-0.8em}
\subfigure[MemNet]{
	\includegraphics[width=0.16\textwidth]{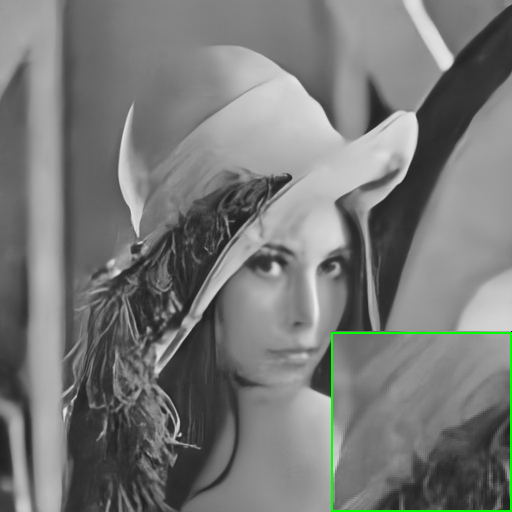}
}\hspace{-0.8em}
\subfigure[Ours]{
    \includegraphics[width=0.16\textwidth]{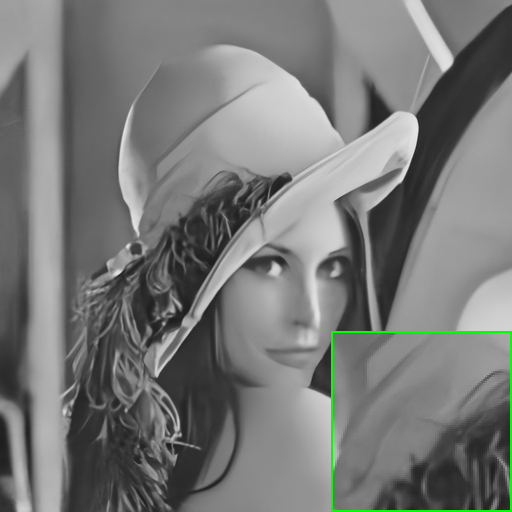}
    }
    \\
         \caption{Denoising results for \emph{Lena} image with noise level 50. (a) Original image  images denoised by  (b) WNNM\cite{WNNM} (29.25 dB); (c) TNRD \cite{TNRD} (28.93 dB); (d) DnCNN-S \cite{Zhang:TIP17} (29.37 dB); (e) MemNet \cite{Tai:ICCV17} (29.63 dB);(f) Ours (\textbf{29.85} dB).}
    \label{fig:den2}

\end{figure*}

\subsection{Image deblurring}

To train the proposed network for image deblurring, we first convoluted the training images with a blur kernel to generate the blurred images and then extracted the training image patches of size $120\times 120$ from the blurred images. The additive Gaussian noise of standard deviation $\sigma_n$ was also added to the blurred images. Patch augmentation with flips and rotations were adopted, generating total $450,000$ patches for training. Two types of blur kernels were considered, i.e., the $25\times 25$ Gaussian blur kernel of standard deviation $1.6$ and two motion blur kernels adopted in \cite{Levin:CVPR09} of sizes $19\times 19$ and $17\times 17$. We trained each model for different blur settings. We compared the proposed method with several leading deblurring methods, i.e., three leading model-based deblurring methods (EPLL \cite{Zoran:ICCV11}, IDDBM3D \cite{IDDBM3D} and NCSR \cite{NCSR}) and the current state-of-the-art denoising-based deblurring method with CNN denoisers \cite{Zhang:CVPR17} (denoted as DD-CNN). We have also compared to the MemNet \cite{Tai:ICCV17} method. We trained MemNet using pairs of blurred image patches and the original image patches. Note that for fair comparisons the same training image patches were used for both the proposed network and MemNet. The test images involved in this comparison study are shown in Fig. \ref{fig:deblur0}. In this experiment, we only conduct deconvolution for grayscale images. However, the proposed method can be easily extended for color image deblurring.

The PSNR results of the test deblurring methods are reported in Table \ref{tab:deblur1}. For fair comparisons, all the PSNRs of the other methods (except MemNet) are generated by the codes released by the authors or directly written according to their papers. From table \ref{tab:deblur1}, we can see that the MemNet method performs much better than conventional model-based EPLL, IDDBM3D and NCSR methods. The proposed method outperforms the MemNet method by up to $0.44$ dB on average. For motion blur kernels with higher noise levels, the proposed method is slightly worse than DD-CNN method that requires much more iterations (up to $30$ iterations) for satisfied results. Parts of the deblurred images by the competing methods are shown in Figs. \ref{fig:deblur1}-\ref{fig:deblur4}. From Figs. \ref{fig:deblur1}-\ref{fig:deblur4}, one can see that the proposed method not only produces more sharper edges but also recovers more details than the other methods.

\begin{figure*}[!tbh]
\renewcommand{\arraystretch}{0.4}
\centering
\subfigure[barbara]{
    \includegraphics[width=0.07\textwidth]{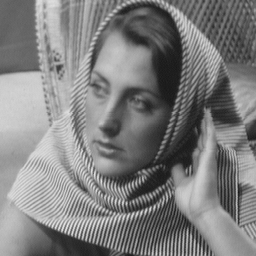}
    }
\subfigure[boats]{
    \includegraphics[width=0.07\textwidth]{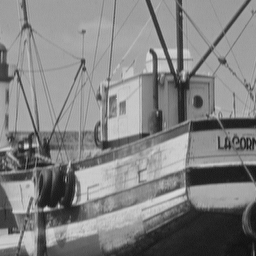}
    }
\subfigure[Butterfly]{
    \includegraphics[width=0.07\textwidth]{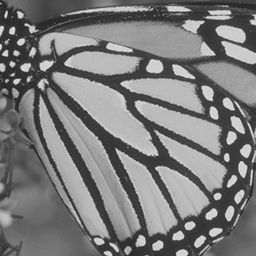}
    }
\subfigure[C.Man]{
    \includegraphics[width=0.07\textwidth]{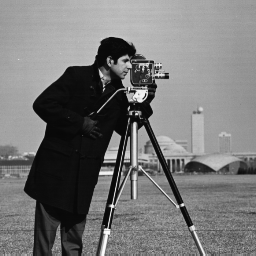}
    }
\subfigure[house]{
    \includegraphics[width=0.07\textwidth]{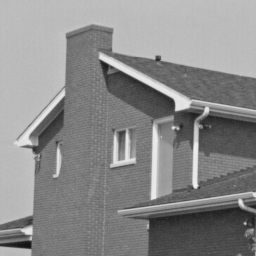}
    }
\subfigure[leaves]{
    \includegraphics[width=0.07\textwidth]{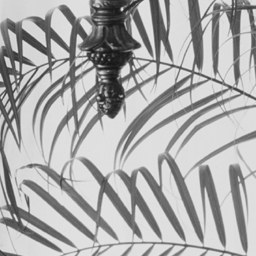}
    }
\subfigure[lena256]{
    \includegraphics[width=0.07\textwidth]{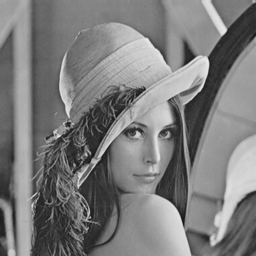}
    }
\subfigure[Parrots]{
    \includegraphics[width=0.07\textwidth]{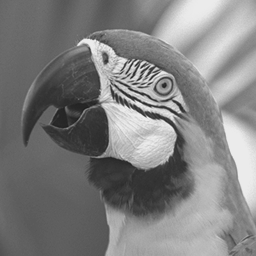}
    }
\subfigure[peppers]{
    \includegraphics[width=0.07\textwidth]{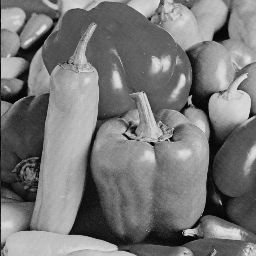}
    }
\subfigure[Starfish]{
    \includegraphics[width=0.07\textwidth]{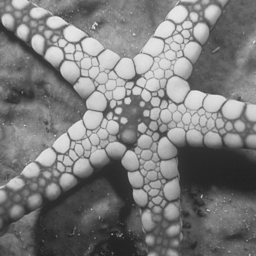}
    }
    \\
     \caption{The test images used for image deblurring.}
     \label{fig:deblur0}
\end{figure*}

\begin{table*}[tbh]
\centering
\caption{The PSNR results of the deblurred images by the test methods. }
\setlength{\tabcolsep}{4.5pt}
\label{tab:deblur1}
\begin{tabular}{!{\vrule width1.2pt}c!{\vrule width1.2pt}c!{\vrule width1.2pt}c|c|c|c|c|c|c|c|c|c|c!{\vrule width1.2pt}}
\Xhline{1.2pt}
Methods  & $\sigma_n$                      & Butterfly & Peppers & Parrot & starfish & Barbara & Boats & C.Man & House  & Leaves & Lena     & Average \\ \Xhline{1.2pt}
\multicolumn{13}{!{\vrule width1.2pt}c!{\vrule width1.2pt}}{Gaussian Blur with standard deviation 1.6}                                                                       \\ \Xhline{1.2pt}
IDD-BM3D & \multirow{5}{*}{2}    &   29.79    &    29.64 &   31.90   &   30.57   &    25.99    &     31.17 &     27.68    &     33.56     &    30.13     &   30.91     &    30.13     \\ \cline{1-1} \cline{3-13}
EPLL     &                       &      25.78     &   26.73    &   31.32    &   28.52    &    24.22    &    28.84  &        26.57 &    31.76      &     25.29    &    29.46    &   27.85      \\ \cline{1-1} \cline{3-13}
NCSR     &                       &     29.72      &     30.04 &   32.07    &     30.83  &   \textbf{26.54}    &   31.22     &   27.99   &     33.38    &    30.13      &   30.99         &   30.29      \\ \cline{1-1} \cline{3-13}
DD-CNN    &                       &     30.44      &   \textbf{30.69}    &   31.83    &   30.78    &       26.15 &   31.41   &    28.05     &   33.80      &  30.44     &   31.05    &   30.48      \\ \cline{1-1} \cline{3-13}
MemNet  &                       & \textbf{30.68}          &  29.79     & \textbf{32.45}     &  31.40    &  25.82     &  31.35    &   28.23      &   33.75      &  \textbf{30.62}     &  31.33     & 30.54        \\ \cline{1-1} \cline{3-13}
Ours     &                       &      30.67     &   30.18  & 32.40    &    \textbf{32.00}   &   26.47    &       \textbf{31.54} &   \textbf{28.24}   &    \textbf{34.25}     &    30.23     &   \textbf{31.48}           &    \textbf{30.75}    \\ \Xhline{1.2pt}

\multicolumn{13}{!{\vrule width1.2pt}c!{\vrule width1.2pt}}{$19\times 19$ motion blur kernel 1 of \cite{Levin:CVPR09}}                                                                                                 \\ \Xhline{1.2pt}
EPLL     & \multirow{3}{*}{2.55} &      26.23     &   27.40    &    33.78   &  29.79     &    29.78    &  30.15    &        30.24 &     31.73     &    25.84     &    31.37    &     29.63    \\ \cline{1-1} \cline{3-13}
DD-CNN    &                       &    32.23       &    32.00   &   34.48    &   32.26    &       32.38 &   33.05   &   31.50      &    34.89      &     33.29    &  33.54      &    32.96     \\ \cline{1-1} \cline{3-13}
MemNet   &                       &   32.19    &  31.67     &  34.47     & 32.47     &  32.10      &  33.15    &   31.29      &   34.57      &   32.16     &  33.40     &   32.75     \\ \cline{1-1} \cline{3-13}
Ours     &                       &     \textbf{32.58}      &   \textbf{32.05}    &  \textbf{34.98}     &   \textbf{32.71}    &       \textbf{32.39} &  \textbf{33.39}    &     \textbf{31.70}    &     \textbf{35.34}     &    \textbf{32.99}     &  \textbf{33.80}      &    \textbf{ 33.19 }    \\ \Xhline{1.2pt}

EPLL     & \multirow{3}{*}{7.65} &     24.27      &   26.15    &   30.01    &  26.81     &    26.95    &  27.72    &        27.37 &     29.89     &     23.81    &   28.69     &    27.17     \\ \cline{1-1} \cline{3-13}
DD-CNN    &                       &   \textbf{28.51}       &   \textbf{28.88}   &  \textbf{31.07}   &  27.86     &       \textbf{28.18} &   29.13   &     \textbf{28.11}    &   32.03       &   \textbf{28.42}      &  \textbf{29.52}     &    \textbf{29.17}     \\ \cline{1-1} \cline{3-13}
MemNet   &               &    28.32     & 28.42     &   30.89    & \textbf{28.02}   &  27.90     &  28.99     &   27.61     &   31.93     &   27.55    &  29.34    &   28.89     \\ \cline{1-1} \cline{3-13}
Ours     &                       &     28.24     &   28.42   &   31.03    & 28.00   &  28.01 & \textbf{29.19}    &       27.77 &   \textbf{32.06}   &    27.98     &    29.42    &   29.01 \\ \Xhline{1.2pt}

\multicolumn{13}{!{\vrule width1.2pt}c!{\vrule width1.2pt}}{ $17\times 17$ motion blur kernel 2 of \cite{Levin:CVPR09}}                                                                                                 \\ \Xhline{1.2pt}
EPLL     & \multirow{3}{*}{2.55} &     26.48      &   27.37    &   33.88    &   29.56    &   28.29     &   29.61   &     29.66    &   32.97       &    25.69     &   30.67     &     29.42    \\ \cline{1-1} \cline{3-13}
DD-CNN    &                       &   \textbf{31.97}       &   \textbf{31.89}    &    34.46   &   32.18    &       \textbf{32.00} &  \textbf{33.06}    &    \textbf{31.29}     &     34.82     &    \textbf{32.96}     &   \textbf{33.35}     &    \textbf{32.80}     \\ \cline{1-1} \cline{3-13}
MemNet   &                       &  31.70       &  30.78    &  34.51     & 31.99    &  30.92     & 32.54      &  30.86      &  34.84      &  31.88     &  33.11    &   32.31     \\ \cline{1-1} \cline{3-13}
Ours     &                      &        31.86    &    31.38         &    \textbf{34.72}   &  \textbf{32.28}     &   31.36    &     32.86   &   31.21   &    \textbf{35.09}     &    32.29      &    \textbf{33.35}     &   32.64      \\ \Xhline{1.2pt}

EPLL     & \multirow{3}{*}{7.65} &     23.85      &   26.04    &   29.99    &    26.78   &   25.47     &   27.46   &    26.58     &     30.49     &    23.42     &    28.20    &     26.83    \\ \cline{1-1} \cline{3-13}
DD-CNN    &                       &      \textbf{28.21}     &   \textbf{28.71}    &   \textbf{30.68}    &   27.67    &       \textbf{27.37} &   \textbf{28.95}   &   \textbf{27.70}      &    \textbf{31.95}      &    \textbf{27.92}     &    \textbf{29.27}    &    \textbf{28.84}     \\ \cline{1-1} \cline{3-13}
MemNet   &                       &   27.55    &   27.60    &  30.57     &  27.55     &  26.53  &  28.68   &   27.28     &  31.61  &   27.02     &  29.10      &   28.35      \\ \cline{1-1} \cline{3-13}
Ours    &                       &      27.47     &   28.02 &    30.46   &   \textbf{27.82}    &   26.86    &    28.84    &    27.48  &    31.91     &    27.28      &    29.23          &         28.54\\ \Xhline{1.2pt}
\end{tabular}
\end{table*}

\begin{figure*}[!tbh]
\renewcommand{\arraystretch}{0.4}
\centering
\subfigure[Original]{
    \includegraphics[width=0.16\textwidth]{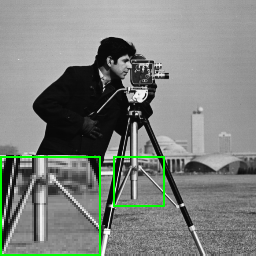}
    }\hspace{-0.8em}

\subfigure[EPLL]{
    \includegraphics[width=0.16\textwidth]{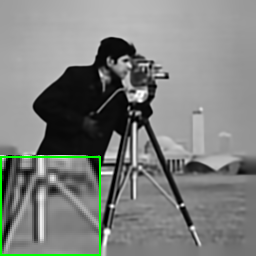}
    }\hspace{-0.8em}
\subfigure[NCSR]{
    \includegraphics[width=0.16\textwidth]{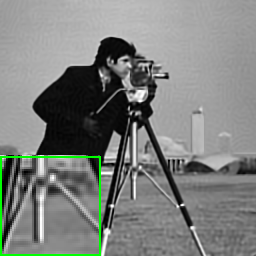}
    }\hspace{-0.8em}
\subfigure[DD-CNN]{
    \includegraphics[width=0.16\textwidth]{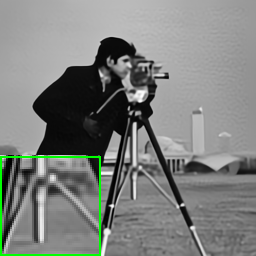}
    }\hspace{-0.8em}
\subfigure[MemNet]{
	\includegraphics[width=0.16\textwidth]{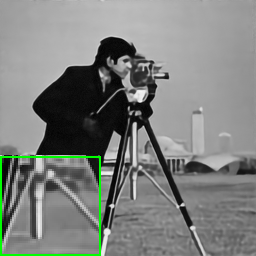}
}\hspace{-0.8em}
\subfigure[Ours]{
    \includegraphics[width=0.16\textwidth]{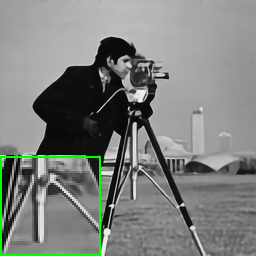}
    }
    \\
        \caption{Deblurring results for \emph{Cameraman} image with $25\times 25$ Gaussian blur kernel and $\sigma_n=2$. (a) Original image; deblurred images by  (b) EPLL denoiser \cite{Zoran:ICCV11} (26.57 dB); (c) NCSR \cite{NCSR} (27.99 dB);(d) DD-CNN \cite{Zhang:CVPR17} (28.05 dB);(e) MemNet \cite{Tai:ICCV17} (28.23 dB); (f) Ours (\textbf{28.24} dB).}
    \label{fig:deblur1}

\end{figure*}

\begin{figure*}[!tbh]
\renewcommand{\arraystretch}{0.4}
\centering
\subfigure[Original]{
    \includegraphics[width=0.19\textwidth]{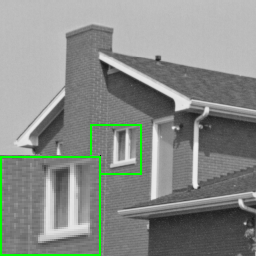}
    }\hspace{-0.8em}
\subfigure[EPLL]{
    \includegraphics[width=0.19\textwidth]{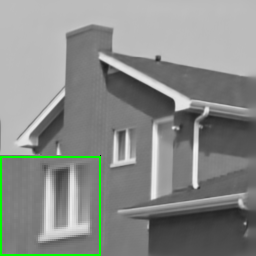}
    }\hspace{-0.8em}
\subfigure[DD-CNN]{
    \includegraphics[width=0.19\textwidth]{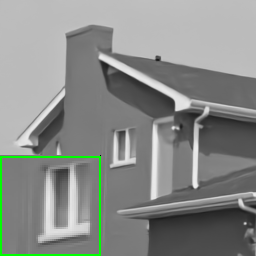}
    }\hspace{-0.8em}
\subfigure[MemNet]{
	\includegraphics[width=0.19\textwidth]{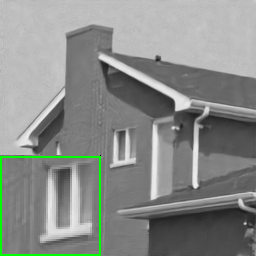}
}\hspace{-0.8em}
\subfigure[Ours]{
    \includegraphics[width=0.19\textwidth]{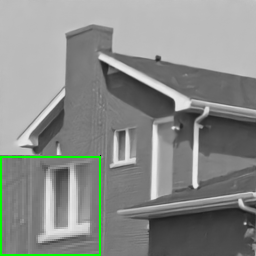}
    }
    \\
     \caption{Deblurring results for \emph{house} image with $19\times 19$ motion blur kernel 1 and $\sigma_n=2.55$. (a) Original image; images deblurred by (b) EPLL \cite{Zoran:ICCV11} (31.73 dB); (c) DD-CNN\cite{Zhang:CVPR17} (34.89 dB);(d) MemNet \cite{Tai:ICCV17} (34.57 dB); (e) Ours (\textbf{35.34} dB).}
     \label{fig:deblur3}
\end{figure*}

\begin{figure*}[!tbh]
\renewcommand{\arraystretch}{0.4}
\centering
\subfigure[Original]{
    \includegraphics[width=0.19\textwidth]{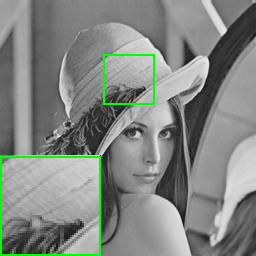}
    }\hspace{-0.8em}
\subfigure[EPLL]{
    \includegraphics[width=0.19\textwidth]{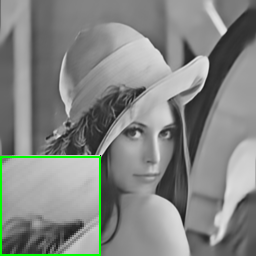}
    }\hspace{-0.8em}
\subfigure[DD-CNN]{
    \includegraphics[width=0.19\textwidth]{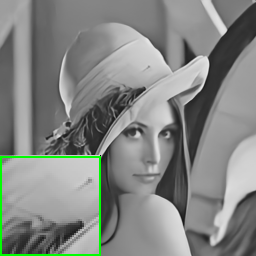}
    }\hspace{-0.8em}
\subfigure[MemNet]{
	\includegraphics[width=0.19\textwidth]{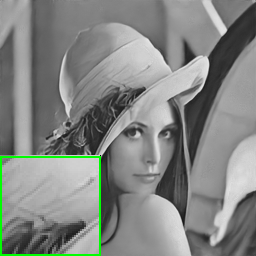}
}\hspace{-0.8em}
\subfigure[Ours]{
    \includegraphics[width=0.19\textwidth]{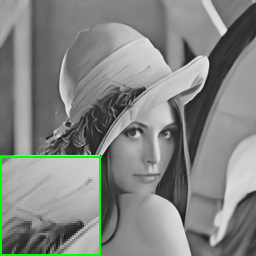}
    }
    \\
     \caption{Deblurring results for \emph{lena} image with $19\times 19$ motion blur kernel 1 and $\sigma_n=2.55$. (a) Original image; images deblurred by (b) EPLL \cite{Zoran:ICCV11} (31.37 dB); (c) DD-CNN\cite{Zhang:CVPR17} (33.54 dB);  (d)MemNet \cite{Tai:ICCV17} (33.40 dB); (d)Ours (\textbf{33.80} dB).}
     \label{fig:deblur4}
\end{figure*}

\subsection{Image super-resolution}

For image super-resolution, we consider two image subsampling operators, i.e., the bicubic downsampling and the Gaussian downsampling. For the former case, the HR images are downsampled by applying the bicubic interpolation function with scaling factor $1/s$ ($s=2, 3, 4$) to simulate the LR images. For the latter case, the LR images are generated by applying the Gaussian blur kernel to the original images followed by subsampling. The $7\times 7$ Gaussian blur kernel of standard deviation of 1.6 is used in this case. The LR/HR patch pairs are extracted from the LR/HR training image pairs and augmented by flip and rotations, generating $450,000$ patch pairs. The LR patch size is $32\times 32$, while the HR patch size is $(32*s)\times (32*s)$. We train each network for the two downsampling cases. The image data sets commonly used in the image SR literature are adopted for performance verification, including the Set5, Set14, BSD100, and the Urban100 dataset \cite{VDSR} containing 100 high-quality images. We compared the proposed method with several leading image SR methods, including two DCNN based SR methods (SRCNN \cite{SRCNN}, VDSR \cite{VDSR} and MemNet \cite{Tai:ICCV17}) and two denoising methods (TNRD \cite{TNRD} and DnCNN \cite{Zhang:TIP17}), which produce the HR images by first upsampling the LR images with the bicubic interpolator and then denoising the upsampled images to recovery the high-frequency details. For fair comparisons, the results of the others are directly borrowed from their papers or generated by the codes released by the authors.

The PSNR results of the test methods for bicubic downsampling are reported in Tables \ref{tab:SR2}. From Table \ref{tab:SR2}, we can see that the proposed method and the MemNet method outperform other methods on average for Set5. The MemNet method is slightly better than the proposed method on average for this dataset. As shown in Table \ref{tab:SR2}, on average the proposed method slightly outperforms the MemNet, which is the second best method in this comparison group. The PSNR results of the test methods for Gaussian downsampling with scaling factor 3 are reported in Table \ref{tab:SR3}. For this case, we compare the proposed method with DD-CNN \cite{Zhang:CVPR17}, which can achieve much better results than their earlier DnCNN \cite{Zhang:TIP17}. We have also compared with the MemNet method \cite{Tai:ICCV17}. For fair comparisons, we also retrained the MemNet model using the pairs of the LR image patches and the corresponding HR image patches. From Table \ref{tab:SR3}, we can see that the proposed method outperforms the MemNet method by larger margins. The proposed method also outperform the iterative DD-CNN method that uses pre-trained DCNN denoiser. Parts of the reconstructed HR images by the test methods are shown in Fig. \ref{fig:SR1}-\ref{fig:SR3}, from which we can see that the proposed method can produce sharper edges than other methods.

\subsection{Complexity anlaysis}

We compared the proposed network with other two state-of-the-art deep learning based IR methods (i.e., the DnCNN \cite{Zhang:TIP17} and the MemNet \cite{Tai:ICCV17}) in terms of complexity. The number of parameters of each deep network are listed in Table \ref{tab:complexity} \footnote{The number of parameters of each network were counted according to their source code downloaded from the authors' website. }. From Table \ref{tab:complexity}, we can see that the MemNet contains the largest number of parameters, almost three times of the proposed network, as it is very deep (up to 80 layers). Since we enforce each denoiser to share the same parameters, the total number of parameters of the proposed network is much smaller than that of MemNet. Though there are $L=6$ stages in the proposed network, the running time of the proposed network is also smaller than that of MemNet. This is due to the fact that the feature maps in the denoiser were gradually downsampled. Thus, the computational complexity can be much reduced.

\begin{table*}[tbh]
	\centering
	\caption{The PSNR and SSIM results of reconstructed HR images by the test methods for the bicubic downsampling.}
	\label{tab:SR2}
	\begin{tabular}{!{\vrule width1.2pt}c!{\vrule width1.2pt}c|c|c|c|c|c|c|c|c|c|c|c|c!{\vrule width1.2pt}}
		\Xhline{1.2pt}
		\multirow{2}{*}{Dataset}  & \multirow{2}{*}{Scaling factor}    & \multicolumn{2}{c|}{TNRD}& \multicolumn{2}{c|}{SRCNN} & \multicolumn{2}{c|}{VDSR} & \multicolumn{2}{c|}{DnCNN} & \multicolumn{2}{c|}{MemNet} & \multicolumn{2}{c!{\vrule width1.2pt}}{Ours} \\ \cline{3-14}
		&                                    & PSNR         & SSIM        & PSNR        & SSIM        & PSNR        & SSIM       & PSNR         & SSIM        & PSNR           & SSIM         & PSNR           & SSIM        \\ \Xhline{1.2pt}
		
		\multirow{3}{*}{Set5}    & 2                                  &  36.86     &    0.956  &  36.66   &  0.954     &  37.53   &  0.959  &  37.58 &  0.959  & \textbf{37.78}   & 0.959  & 37.75     &  \textbf{0.960}   \\ \cline{2-2} \cline{2-14}
		& 3                                  &  33.18    &  0.915  &  32.75   & 0.909  &  33.66  & 0.921    &  33.75   &  0.922   & \textbf{34.09}  & \textbf{0.925}   &   33.93     & 0.924     \\ \cline{2-2} \cline{2-14}
		& 4                                  &  30.85     &  0.873   &  28.42  &  0.810   &  31.35    &  0.884  &    31.40 &  0.885  & \textbf{31.74}    &  \textbf{0.889}     & 31.72    & \textbf{0.889}    \\ \cline{2-14}  \Xhline{1.2pt}
		
		\multirow{3}{*}{Set14}    & 2                                  & 32.54        &0.907        & 32.42       & 0.906       &33.03        &  0.912     & 33.03        &  0.911      &  33.28   & 0.914      & \textbf{33.30}       &    \textbf{0.915}\\ \cline{2-2} \cline{2-14}
		& 3                                  &      29.46   &    0.823    &   29.28     &   0.821     &   29.77     & 0.831      &    29.82     &  0.830      &  30.00   & 0.835     &    \textbf{30.02}    &    \textbf{0.836 }   \\ \cline{2-2} \cline{2-14}
		& 4                                  &     27.68    &     0.756   &  27.49      &  0.750      &   28.01     &  0.767     &27.83         &  0.755      &  28.26   & 0.772     &   \textbf{ 28.28 }   &  \textbf{0.773}   \\ \cline{2-14}  \Xhline{1.2pt}
		
		\multirow{3}{*}{BSD100}   & 2                                  &      31.40   &   0.888     &     31.36    &    0.888   &     31.90   &0.896       &   31.84      & 0.894       &  32.08   & \textbf{ 0.900}     &    \textbf{ 32.09 }  &   0.899\\ \cline{2-2} \cline{2-14}
		& 3                                  &     28.50    &     0.788   &      28.41   &    0.786   &    28.82    &   0.798    &   28.80      &    0.795    &  28.96   &  0.800    &     \textbf{29.00}   & \textbf{0.801} \\ \cline{2-2} \cline{2-14}
		& 4                                  &      27.00   &    0.714    &      26.90   &    0.710   &    27.29    &   0.725    &  27.08       &  0.709      &  27.40   &  0.728    &    \textbf{ 27.44 }  & \textbf{0.729} \\ \cline{2-14} \Xhline{1.2pt}
		
		\multirow{3}{*}{Urban100} & 2                                  &      29.70   &     0.899   &     29.50    &     0.895  &    30.76    &   0.914    &   30.63      &  0.911      &  31.31   &  0.920    &   \textbf{ 31.50 }   &     \textbf{0.922}\\ \cline{2-2} \cline{3-14}
		& 3                                  &      26.44   &     0.807   &     26.24    &    0.799   &     27.14   &    0.828   &    27.08     &   0.824     &  27.56   & 0.838    &    \textbf{27.61}    &     \textbf{0.842}\\ \cline{2-2} \cline{3-14}
		& 4                                  &     24.62    &     0.729   &     24.52    &    0.722   &     25.18   & 0.752      &   24.94      &     0.735   &  25.50   & 0.763     &     \textbf{ 25.53 } &   \textbf{0.768} \\ \cline{2-14} \Xhline{1.2pt}
	\end{tabular}
\end{table*}

\begin{table*}[tbh]
	\centering
	\caption{The PSNR results of the reconstructed HR images by the test methods for the Gaussian downsampling with scaling factor 3.}
	\begin{tabular}{|p{1cm}<{\centering}|p{1.3cm}<{\centering}|p{1.3cm}<{\centering}|p{1.3cm}<{\centering}|p{1.3cm}<{\centering}|}
		\hline
		Dataset & NCSR  & DD-CNN & \begin{tabular}[c]{@{}c@{}} MemNet\end{tabular} & Ours \\ \hline
		Set5    & 32.07  & 33.88  &   33.75     &  \textbf{34.22}    \\ \hline
		Set14   & 29.30  & 29.63  &   29.44      &  \textbf{29.88}    \\ \hline
	\end{tabular}	\label{tab:SR3}

\end{table*}

\begin{table}[]
		\centering			
		\caption{Complexity comparison with other deep networks. The average running time for an image is measured on a Nvidia Titan XP GPU for the denoising task on the BSD68 dataset for noise level 50.}
	\begin{tabular}{|c|c|c|c|c|}
		\hline
		Method      & DnCNN & MemNet &  DPDNN$_{6}$   \\ \hline
		\#Paras     & 665K  & 2892K  &  1066K \\ \hline
		Run Time(s/image) & 0.048 & 0.278 &  0.161 \\ \hline
		PSNR(dB)    & 26.24 & 26.37  &  27.02 \\ \hline
	\end{tabular} \label{tab:complexity}
\end{table}

\begin{figure*}[!tbh]
\renewcommand{\arraystretch}{0.4}
\centering
\subfigure[Original]{
    \includegraphics[width=0.16\textwidth]{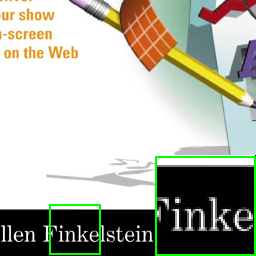}
    }\hspace{-0.8em}
\subfigure[TNRD]{
    \includegraphics[width=0.16\textwidth]{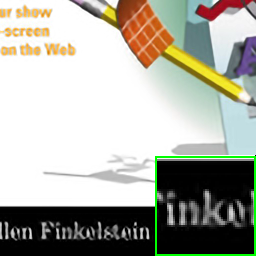}
    }\hspace{-0.8em}
\subfigure[VDSR]{
    \includegraphics[width=0.16\textwidth]{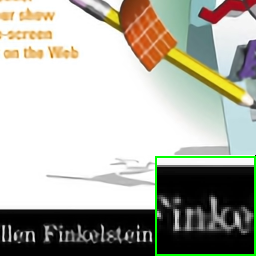}
    }\hspace{-0.8em}
\subfigure[DnCNN]{
    \includegraphics[width=0.16\textwidth]{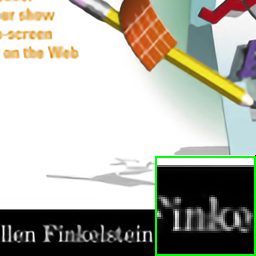}
    }\hspace{-0.8em}
\subfigure[MemNet]{
	\includegraphics[width=0.16\textwidth]{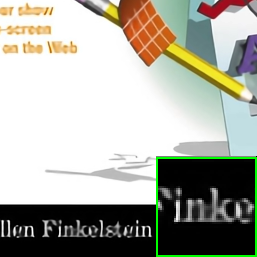}
}\hspace{-0.8em}
\subfigure[Ours]{
    \includegraphics[width=0.16\textwidth]{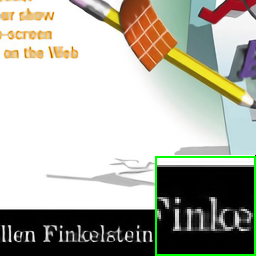}
    }
    \\
        \caption{SR results for \emph{13th} image of Set14 for bicubic downsampling and scaling factor 3. The PSNR results:(b) TNRD \cite{TNRD} (27.08 dB);(c) VDSR \cite{VDSR} (27.86 dB); (d) DnCNN \cite{Zhang:TIP17} (28.21 dB); (e) MemNet \cite{Tai:ICCV17} (28.92 dB);  (f) Ours (\textbf{28.99} dB).}
    \label{fig:SR1}

\end{figure*}

\begin{figure*}[!tbh]
\renewcommand{\arraystretch}{0.4}
\centering
\subfigure[Original]{
    \includegraphics[width=0.16\textwidth]{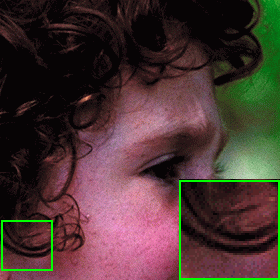}
    }\hspace{-0.8em}
\subfigure[TNRD]{
    \includegraphics[width=0.16\textwidth]{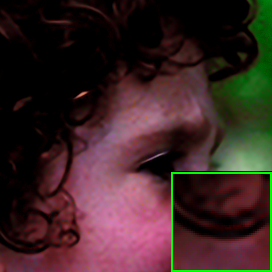}
    }\hspace{-0.8em}
\subfigure[VDSR]{
    \includegraphics[width=0.16\textwidth]{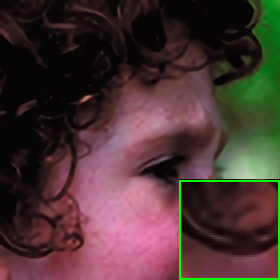}
    }\hspace{-0.8em}
\subfigure[DnCNN]{
    \includegraphics[width=0.16\textwidth]{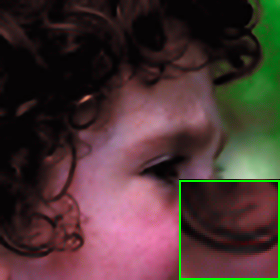}
    }\hspace{-0.8em}
\subfigure[MemNet]{
	\includegraphics[width=0.16\textwidth]{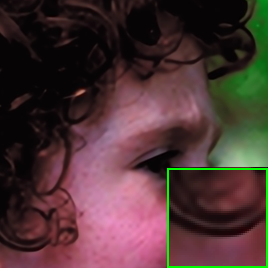}
}\hspace{-0.8em}
\subfigure[Ours]{
    \includegraphics[width=0.16\textwidth]{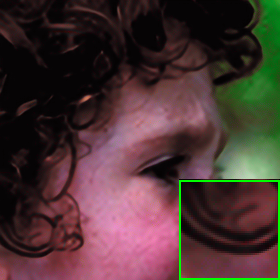}
    }
    \\
        \caption{Results for \emph{6th} image of Set14 for bicubic downsampling and scaling factor 4. The PSNR results: (b)TNRD \cite{TNRD} (32.51 dB);  (c)VDSR \cite{VDSR} (32.70 dB); (d)DnCNN\cite{Zhang:TIP17} (32.36 dB); (e)MemNet\cite{Tai:ICCV17} ({32.79} dB); (f)Ours (\textbf{32.88} dB).}
    \label{fig:SR2}

\end{figure*}

\begin{figure*}[!tbh]
\renewcommand{\arraystretch}{0.4}
\centering
\subfigure[Original]{
    \includegraphics[width=0.188\textwidth]{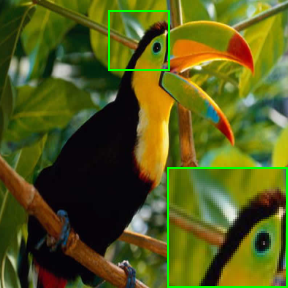}
    }\hspace{-0.8em}
\subfigure[NCSR]{
    \includegraphics[width=0.188\textwidth]{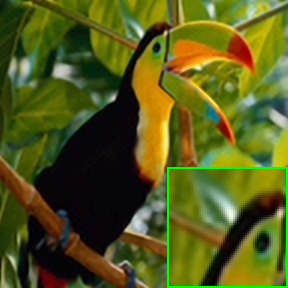}
    }\hspace{-0.8em}
\subfigure[DD-CNN]{
    \includegraphics[width=0.188\textwidth]{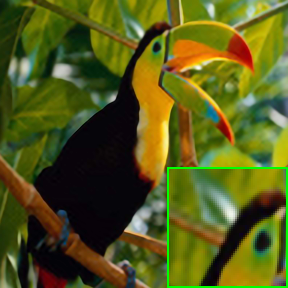}
    }\hspace{-0.8em}
\subfigure[MemNet]{
	\includegraphics[width=0.188\textwidth]{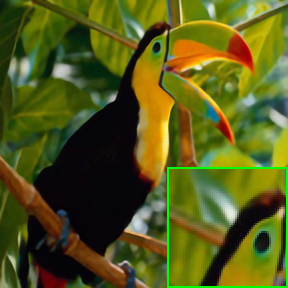}
}\hspace{-0.8em}
\subfigure[Ours]{
    \includegraphics[width=0.188\textwidth]{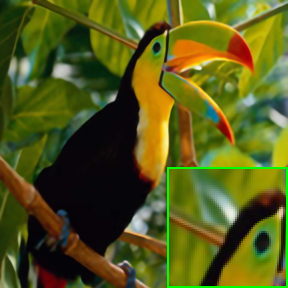}
    }
    \\
    \caption{SR results for \emph{2th} image of Set5 for Gaussian downsampling with scaling factor 3. The PSNR results: (b) NCSR \cite{NCSR} (35.74 dB);  (c) DD-CNN \cite{Zhang:CVPR17} (35.93 dB); (d) MemNet\cite{Tai:ICCV17} (36.92 dB); (e) Ours (\textbf{37.75} dB).}
    \label{fig:SR3}

\end{figure*}

\section{Conclusion}

In this paper, we have proposed a novel deep neural network for general image restoration (IR) tasks. Different from current deep network based IR methods, where the observation models are generally ignored, we construct the deep network based on a denoising-based IR framework. To this end, we first developed an efficient algorithm for solving the denoising-based IR method and then unfolded the algorithm into a deep network, which is composed of multiple denoising modules interleaved with back-projection modules for data consistencies. A DCNN-based denoiser exploiting multi-scale redundancies of natural images was developed. Therefore, the proposed deep network can exploit not only the effective DCNN denoising prior but also the prior of the observation model. Experimental results show that the proposed method can achieve very competitive and often state-of-the-art results on several IR tasks, including image denoising, deblurring and super-resolution.

\appendix
\section*{Convergence}
\begin{thm}[] Consider the energy function
$$\xi(\tx,\tv) := \frac{1}{2}\|\ty-\tA\tx\|_2^2 + \frac{\eta}{2}\|\tx-\tv\|_2^2 + \lambda J(\tv).$$
Assume that $\xi$ is lower bounded and coercive\footnote{$\xi(\tx,\tv)\to\infty$ whenever $\|(\tx,\tv)\|\to\infty$.}.
For Algorithm 1, $(\tx^{(t)}, \tv^{(t)})$ has a subsequence that converges to a stationary point of the the energy function provided that the denoiser $f(\cdot)$ satisfies the sufficient descent condition:
\begin{align}\label{fcond1}
&\frac{\eta}{2}||\tx-\tv||_2^2 + \lambda J(\tv)-\frac{\eta}{2}||\tx-f(\tx)||_2^2 - \lambda J(f(\tx))\nonumber\\
&\ge c_2\|\tilde\nabla_{\tv}\xi(\tx,\tv)\|_2^2,
\end{align}
where $c_2>0$ and $\tilde\nabla_{\tv}\xi(\tx,\cdot)$ is a continuous limiting subgradient of $\xi$.
\end{thm}
\begin{proof}
Since $\nabla_{\tx}\xi(\tx,\tv)$ is Lipschitz continuous with constant $\|\tA^{\top}\tA\|+\eta$, it is well known that the gradient step on $\tx$ with step size $\delta\in(0,\frac{2}{\|\tA^{\top}\tA\|+\eta})$ satisfies the descent property
\begin{align}\label{desx}
\xi(\tx^{(t)},\tv^{(t)})-\xi(\tx^{(t+1)},\tv^{(t)}) \ge c_1 \|\tx^{(t)}-\tx^{(t+1)} \|_2^2, 
\end{align}
where $c_1:=\frac{1}{\delta}-\frac{\|\tA^{\top}\tA\|+\eta}{2}>0$.
By assumption, the $\tv$-step satisfies
\begin{align}\label{desv}
\xi(\tx^{(t+1)},\tv^{(t)})-\xi(\tx^{(t+1)},\tv^{(t+1)}) \ge c_2\|\tilde\nabla_{\tv}\xi(\tx^{(t+1)},\tv^{(t)})\|_2^2.
\end{align}
Since $\xi(\tx,\tv)$ is coercive and, by \eqref{desx} and \eqref{desv}, $\xi(\tx^{(t)},\tv^{(t)})$ is monotonically nonincreasing, the sequence $(\tx^{(t)},\tv^{(t)})_{t=0,1,2,\ldots}$ is bounded (otherwise, it would cause the contradiction $\xi(\tx^{(t)},\tv^{(t)})\to \infty$), so it has a convergent subsequence $(\tx^{(t_k)},\tv^{(t_k)})_{k=0,1,\ldots}\overset{k}{\to}(\tx^{*},\tv^{*})$.
Since $\xi(\tx,\tv)$ is lower bounded, adding (9) and (10) yields
\begin{align}
&\xi(\tx^{(t)},\tv^{(t)})-\xi(\tx^{(t+1)},\tv^{(t+1)})\nonumber\\
&\ge c_1 \|\tx^{(t)}-\tx^{(t+1)} \|_2^2+c_2\|\tilde\nabla_{\tv}\xi(\tx^{(t+1)},\tv^{(t)})\|_2^2.
\end{align}
and, by telescopic sum over $t=0,1,\ldots$ and  by monotonicity and boundedness of $\xi(\tx^{(t)},\tv^{(t)})$, we have the summability properties $\sum_t\|\tx^{(t)}-\tx^{(t+1)} \|_2^2 <\infty$ and $\sum_t\|\tilde\nabla_{\tv}\xi(\tx^{(t+1)},\tv^{(t)})\|_2^2 <\infty$, from which we conclude
\begin{align}
&\lim_{t\to\infty}\|\tx^{(t)}-\tx^{(t+1)} \|_2=0,\label{txcvg}\\
&\lim_{t\to\infty}\|\tilde\nabla_{\tv}\xi(\tx^{(t+1)},\tv^{(t)})\|_2 = 0.
\end{align}
Based on $\tx^{(t+1)}-\tx^{(t)}=\delta \nabla_{\tx} \xi(\tx^{(t)},\tv^{(t)})$, we get
$\nabla_{\tx} \xi(\tx^{*},\tv^{*}) =\lim_{k\to\infty}\nabla_{\tx} \xi(\tx^{(t_k)},\tv^{(t_k)}) =0$, where we have used the continuity of  $\nabla_{\tx} \xi(\tx,\tv)$ in $\tx$. Also, $\lim_{k\to\infty}\tilde\nabla_{\tv} \xi(\tx^{(t_k)},\tv^{(t_k)}) =\lim_{k\to\infty}\tilde\nabla_{\tv} \xi(\tx^{(t_k+1)},\tv^{(t_k)})=0$, where the first ``$=$'' follows from the continuity of $\nabla_{\tv} \xi(\tx,\tv)=2\mu(\tv-\tx)+\lambda\tilde \nabla J(\tv)$ in $\tx$ and \eqref{txcvg}. Therefore, $(\tx^{*},\tv^{*})$ is a stationary point of $\xi$.
\end{proof}

\ifCLASSOPTIONcaptionsoff
  \newpage
\fi

\bibliographystyle{splncs}
\bibliography{DNN}

\end{document}